\documentclass{article}
\usepackage[final,main]{neurips_2026}
\setcitestyle{numbers,square,sort&compress}

\makeatletter
\renewcommand{\@toptitlebar}{%
  \hrule height 3pt
  \vskip 1pt
  \vskip 0.25in
  \vskip -\parskip
}
\makeatother

\usepackage[utf8]{inputenc}
\usepackage[T1]{fontenc}
\usepackage{url}
\usepackage{booktabs}
\usepackage{hyperref}
\usepackage{amsfonts}
\usepackage{amsmath}
\usepackage{amssymb}
\usepackage{nicefrac}
\usepackage{microtype}
\usepackage{xcolor}
\usepackage{graphicx}
\usepackage{eso-pic}
\usepackage{subcaption}
\usepackage{algorithm}
\usepackage{algpseudocode}
\usepackage{amsthm}
\usepackage{multirow}
\usepackage{wrapfig}
\usepackage{enumitem}
\usepackage{xspace}
\usepackage{colortbl}
\usepackage{adjustbox}
\usepackage{arydshln}
\definecolor{bluecite}{HTML}{0071BC}
\definecolor{tablerowcolor}{RGB}{220,235,245}

\newcommand{\method}{PACE\xspace}
\newcommand{\ie}{\textit{i}.\textit{e}.}
\newcommand{\eg}{\textit{e}.\textit{g}.}

\newtheorem{theorem}{Theorem}[section]
\newtheorem{corollary}[theorem]{Corollary}
\newtheorem{proposition}[theorem]{Proposition}

\theoremstyle{remark}
\newtheorem*{takeaway}{Takeaway}

\newcommand{\pos}[1]{\textbf{\textcolor{teal}{#1}}}    
\newcommand{\negitive}[1]{\textcolor{red}{#1}}      

\newcommand{\positive}[1]{\textcolor{teal}{#1}}
\newcommand{\negative}[1]{\textcolor{red!70!black}{#1}}
\newcommand{\real}[1]{{\scriptsize\textcolor{gray}{#1}}}
\title{PACE: \underline{P}areto-\underline{A}daptive \underline{C}ompression for \underline{E}fficient Native MLLMs}

\author{%
  Lei Tian\textsuperscript{1}\thanks{Equal contribution.\quad \textsuperscript{\ensuremath{\ddagger}}Corresponding author.}\quad
  Chunzheng Zhu\textsuperscript{2}\footnotemark[1]\hspace{0.55em}\thanks{Work done during an internship at ByteDance.}\quad
  Kai WU\textsuperscript{1}\\
  \bfseries YiFan Zhang\textsuperscript{3}\quad
  Haihua Yang\textsuperscript{1}\textsuperscript{\ensuremath{\ddagger}}\\[0.5em]
  \normalfont \textsuperscript{1}ByteDance\\
  \normalfont \textsuperscript{2}Hunan University\\
  \normalfont \textsuperscript{3}Institute of Automation, Chinese Academy of Sciences
}

\begin{document}
\AddToShipoutPictureFG*{%
  \AtPageUpperLeft{%
    \put(\LenToUnit{\dimexpr1in+\oddsidemargin\relax},\LenToUnit{-23mm}){%
      \includegraphics[height=10mm]{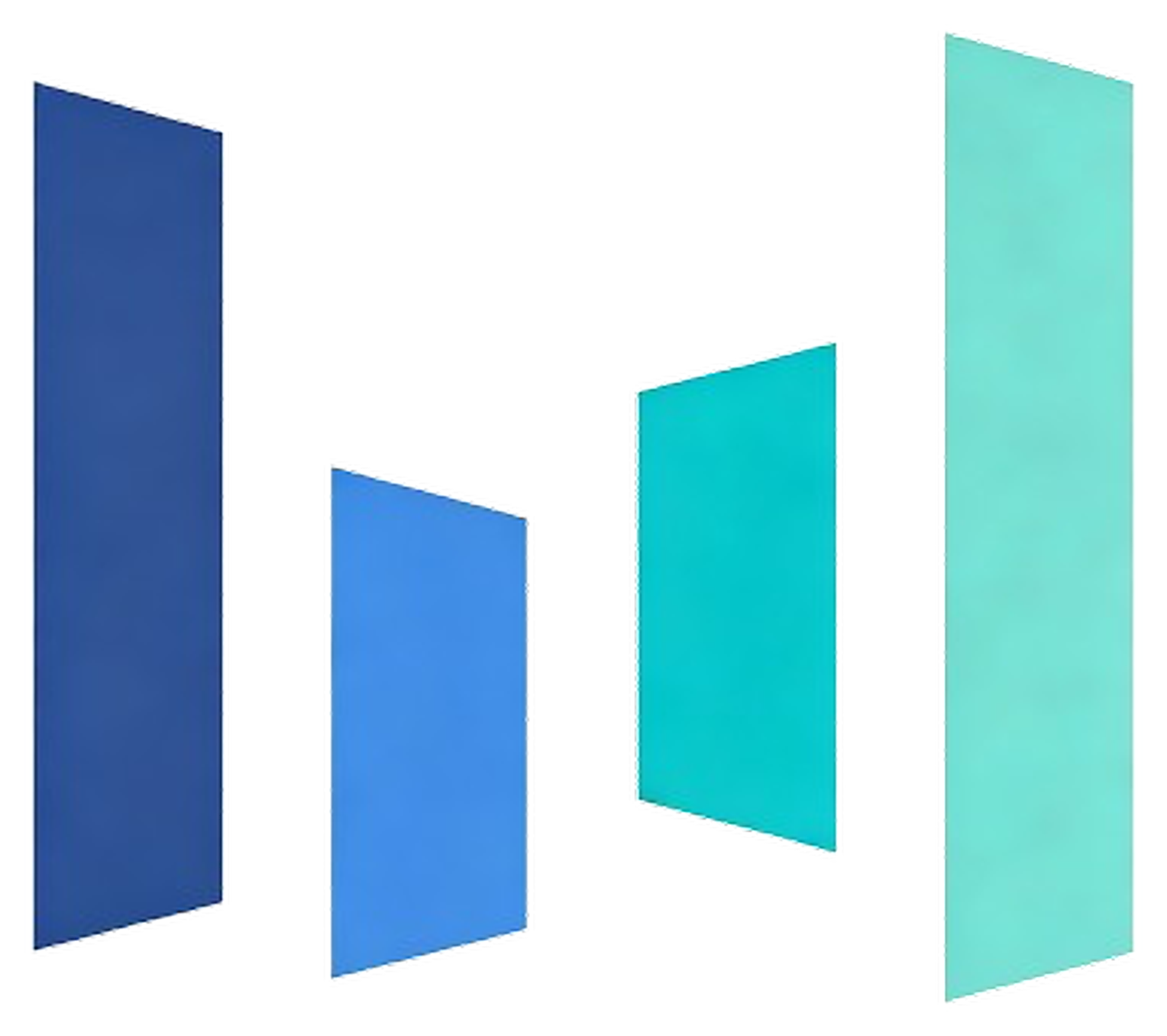}%
      \hspace{4mm}%
      \includegraphics[height=10mm]{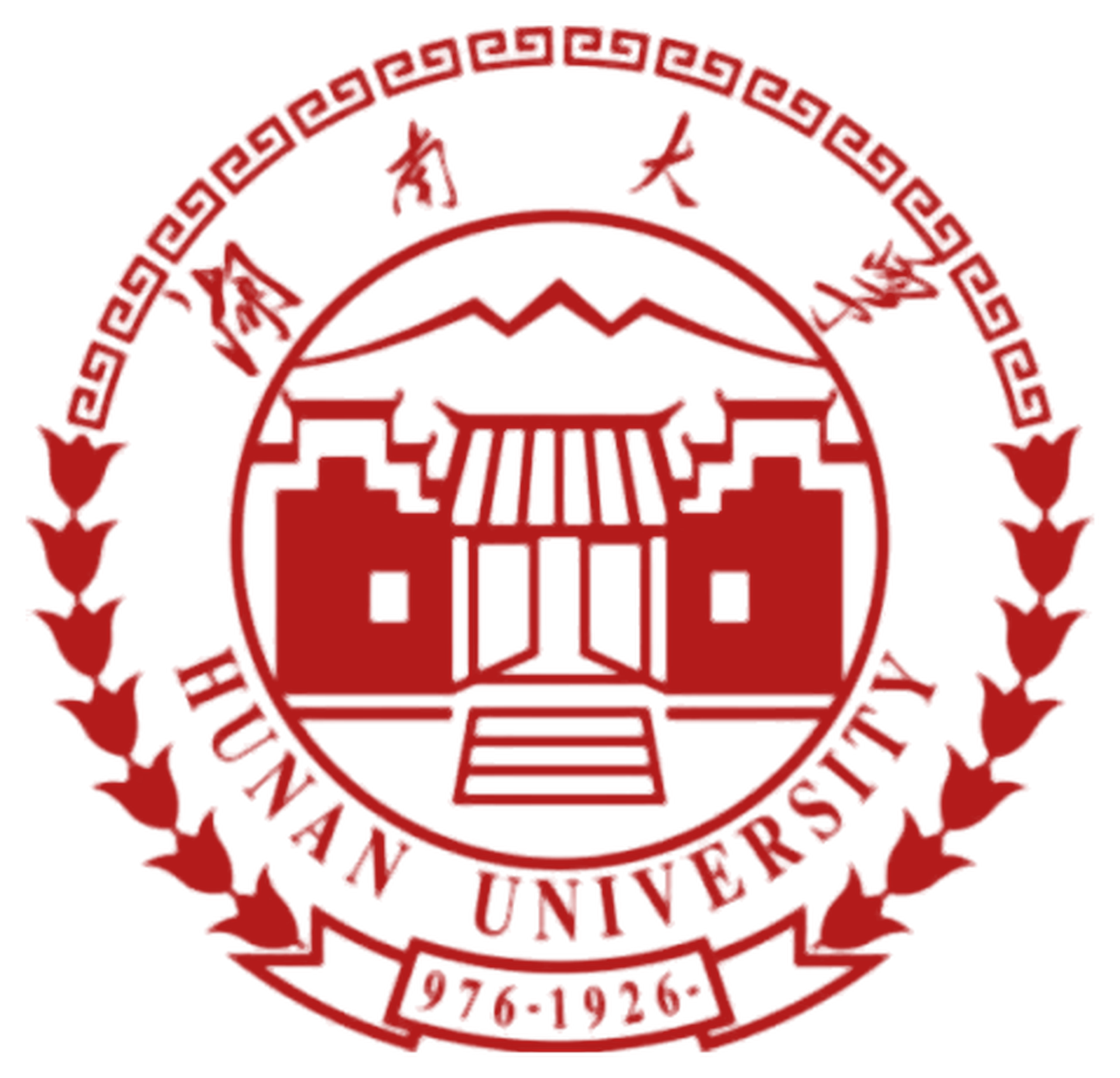}%
      \hspace{4mm}%
      \includegraphics[height=10mm]{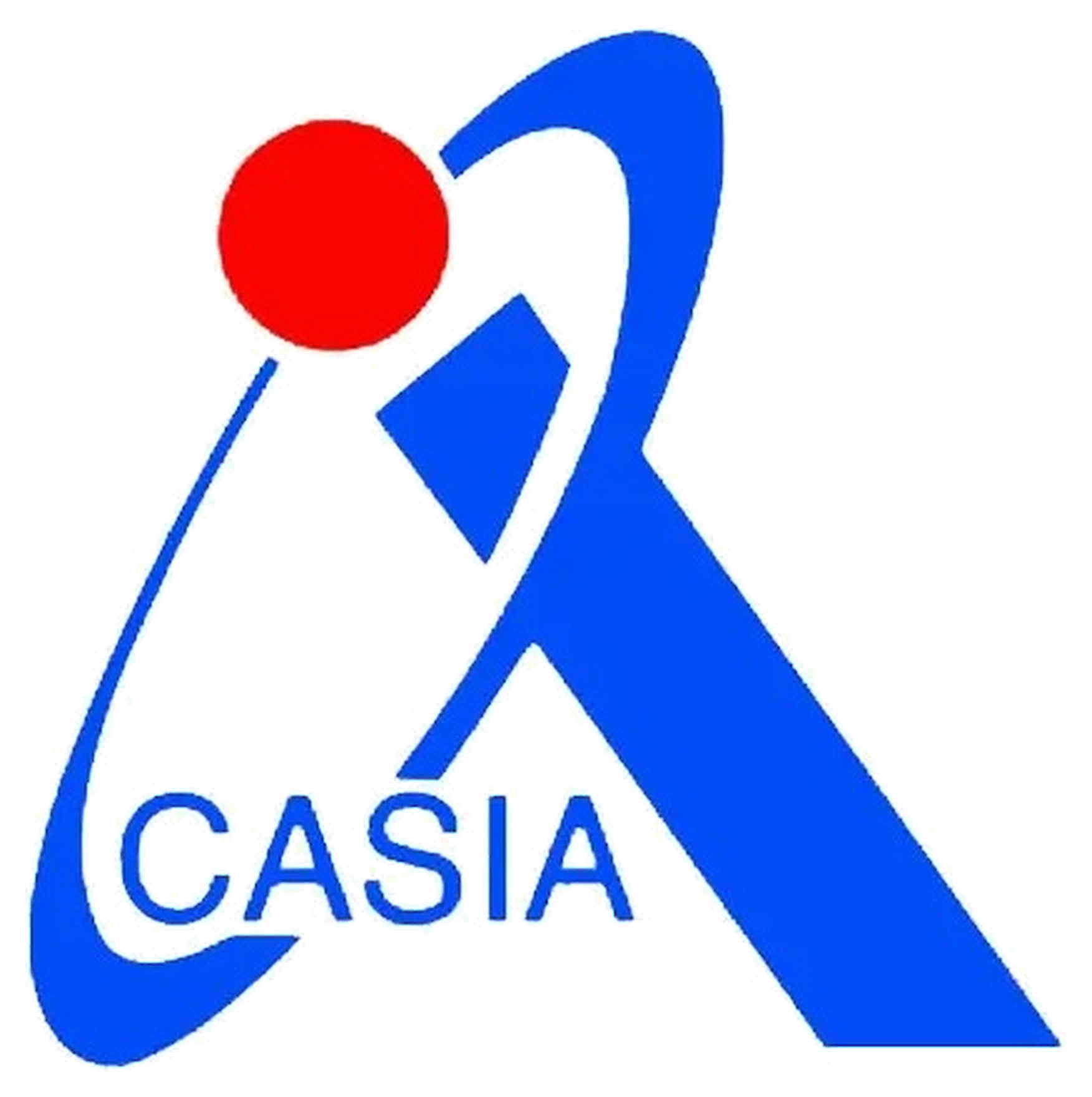}%
    }%
  }%
}

\maketitle

\begin{abstract}
Native multimodal large language models (native MLLMs) jointly encode visual and linguistic sequences within a unified autoregressive Transformer, emerging as a promising paradigm for multimodal understanding. Unlike conventional ViT-based MLLMs that rely on an external vision encoder, ViT-free native MLLMs tokenize every pixel patch directly, which is intuitively expected to suffer from severe visual token redundancy. A pilot study confirms this: unlike ViT-based models that degrade under compression, ViT-free MLLMs tolerate aggressive compression with near-lossless or even improved performance. Yet existing compression methods target only the ViT-based setting, leaving native MLLMs, which stand to gain the most, entirely unexplored. To close this gap, we introduce an analysis pipeline for native backbones centered on two metrics, \textit{i.e.}, token discriminability and cross-dimensional heterogeneous redundancy, revealing that \textit{(i)} redundancy is distributed independently and heterogeneously along the height and width axes, and \textit{(ii)} early-layer visual tokens are indiscriminable and uniformly redundant. Guided by these findings, we propose \textbf{PACE} (\underline{P}areto \underline{A}daptive \underline{C}ompression for \underline{E}fficient native MLLM), the first visual token compression framework tailored for native MLLMs with Multimodal-RoPE. Leveraging the dimensional decoupling of position encoding, PACE decomposes attention into two orthogonal spatial saliency estimates along height and width, gated by a positional attention signal, yielding a direction-aware importance measure with negligible overhead. It further casts token budget allocation as a Pareto optimization over performance preservation, inference efficiency, and spatial coverage, jointly identifying the optimal compression depths and retention ratios. Experiments show that PACE maintains near-lossless performance at only 25\%–50\% visual token retention, and even improves on several benchmarks, while delivering over 2× FLOPs reduction.
\end{abstract}

\section{Introduction}
\label{sec:introduction}

\begin{figure}[t]
  \centering
  \includegraphics[width=\linewidth]{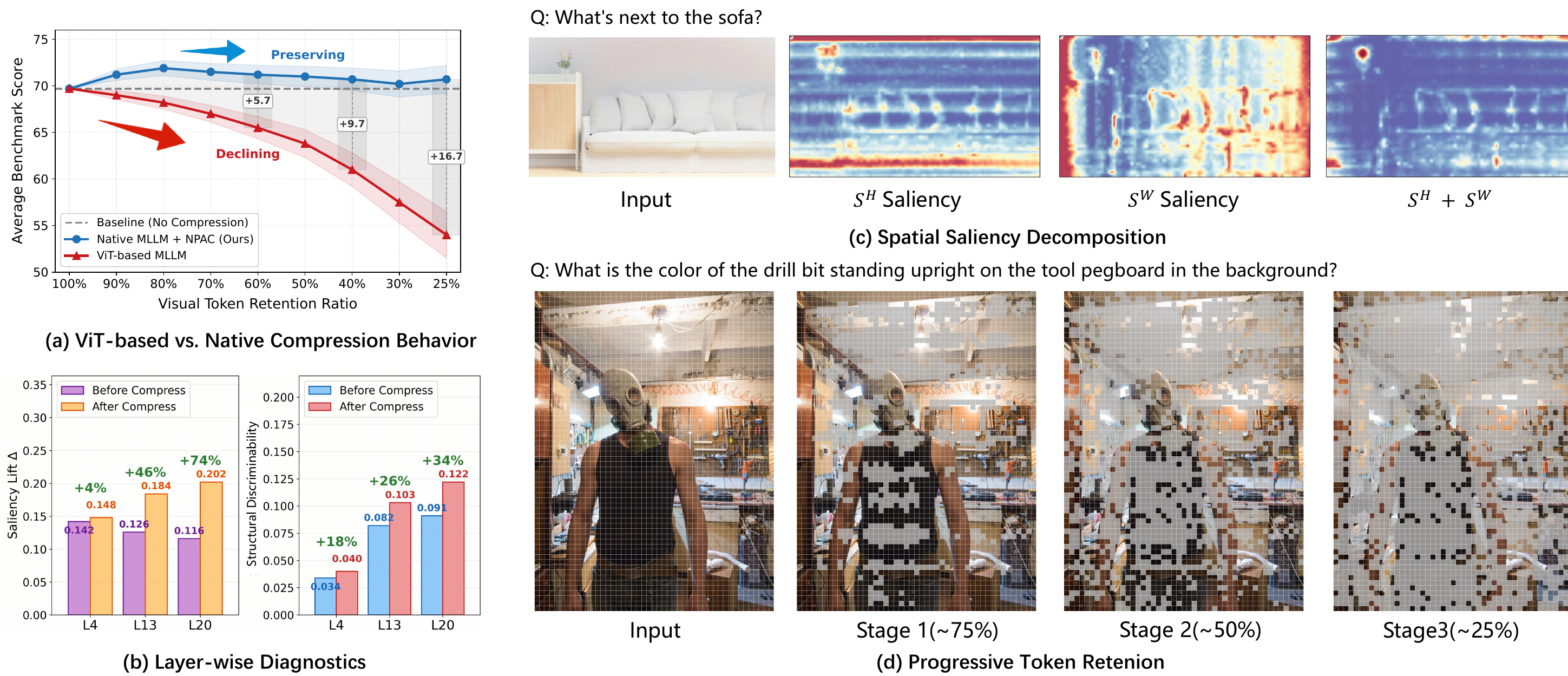}
\caption{\textbf{Probing analysis and progressive compression in native MLLMs.} (a)~ViT-based MLLMs degrade consistently under compression, while native MLLMs could preserve or even improve performance. (b)~Layer-wise discriminability reveals a depth-progressive ramp after the IFB, guiding compression stage placement. (c)~Spatially decoupled saliency along height and width dimensions. (d)~Progressive token retention across multiple compression stages.}
\label{fig:teaser}
\end{figure}

Native multimodal large language models (MLLMs) such as Emu3~\citep{wang2024emu3}, Chameleon~\citep{team2024chameleon}, SOLO~\citep{chen2024solo}, SAIL~\citep{sail25} and NEO~\citep{diao2025pixels} jointly model pixel and token sequences within a unified autoregressive Transformer.
Unlike ViT-based MLLMs~\citep{liu2023visual,liu2024improved,chen2024internvl,zhu2025internvl3,team2023qwen,wang2024qwen2} that rely on an external vision encoder, they tokenize every pixel patch directly, and intuitively suffer from severe visual token redundancy that degrades performance and bottlenecks inference efficiency.
A pilot study on MMBench-DEV-EN~\citep{liu2024mmbench} (Figure~\ref{fig:teaser}(a)) comparing Qwen3-VL~\citep{bai2025qwen3} against a comparable native MLLM built in this work (\S\ref{sec:method}) confirms this: unlike ViT-based models that degrade under compression, native MLLMs tolerate aggressive compression near-losslessly or even improve.
Yet existing token-compression techniques, from attention-guided dropping~\citep{chen2024image,xing2024pyramiddrop,wu2026hidrop} and token merging~\citep{shang2025llava,li2025tokenpacker,yang2025visionzip} to diversity-aware selection~\citep{zhang2024sparsevlm,alvar2025divprune,ye2025fit,zou2025don}, universally presuppose a modular ViT-based architecture, leaving native MLLMs, which should benefit most, entirely unexplored.

To solve this gap, we first conduct a systematic \textit{probing analysis} of token redundancy in representative native backbones, centered on two quantitative metrics. \textit{Discriminability} is assessed via two probes at each compression stage: (a) \textit{Structural Discriminability}, the centered effective rank of the visual token matrix divided by the token count, where higher values indicate tokens span more independent directions with less redundancy; and (b) \textit{Saliency Lift}, the relative gain in per-token 3D decomposed attention importance (computed along temporal, height, and width axes and then fused) from the full set to the PACE-selected subset, where higher values indicate the subset concentrates substantially more attention contribution than the layer-wise average. \textit{Cross-dimensional heterogeneous redundancy} measures how unevenly redundancy is distributed along spatial height and width axes. Based on the analysis of these metrics (Figure~\ref{fig:teaser}(b) and (c)), we reveal two key findings.

\textbf{Finding 1: Initial fusion buffer and depth-progressive redundancy accumulation.} The initial layers serve as an \textit{Initial Fusion Buffer (IFB)}~\citep{diao2025pixels}, a concept established in prior native-MLLM analysis; our contribution is to identify its implication for token compression rather than to establish a new mechanism. During the IFB, visual and language tokens undergo semantic alignment through shared attention, so a stable discriminative structure has not yet emerged and Structural Discriminability in Figure~\ref{fig:teaser}(b) stays low regardless of compression. Its monotonic rise with depth, mirrored by the growth of Saliency Lift $\Delta$ on the right of Figure~\ref{fig:teaser}(b), forms a depth-progressive ramp that justifies conservative pruning near the IFB and aggressive compression on the deeper high-discriminability plateau, motivating the multi-stage schedule in PACE.

\textbf{Finding 2: Heterogeneous redundancy across spatial dimensions.} We observe empirically that redundancy is distributed independently and heterogeneously along the height and width dimensions: certain tokens are redundant along one axis yet critical along the orthogonal one. For instance, horizontal text lines in documents are salient along height but uniform along width, while vertical structures show the opposite. A single attention score conflating all dimensions cannot distinguish these modes~\citep{wang2024qwen2,zou2025don}. We treat this as empirical motivation for our spatially decoupled importance estimation, which separately evaluates height and width saliency, rather than as a broadly established property (visualization in Figure~\ref{fig:teaser}(c)).

Motivated by these findings, we propose \textbf{\method} (Pareto-Adaptive Compression for Efficient
Native MLLMs), the first visual token compression framework for native MLLMs with multi-axis positional encoding, comprising three components: (1)~\textbf{Spatially Decoupled Importance Estimation (SDIE)} (\S\ref{sec:sdie}) exploits the dimensional decoupling of Native-RoPE to split attention into height and width saliency gated by a positional signal, giving a direction-aware, parameter-free importance measure. (2)~\textbf{Pareto Adaptive Compression Scheduling (PACS)} (\S\ref{sec:pacs}) uses a Pareto front search to locate compression layers along the discriminability curve and to allocate the cross-layer budget over performance, efficiency, and spatial coverage, yielding the per-stage retention ratio. (3)~\textbf{Four-stage end-to-end training} (\S\ref{sec:training}) builds a full PT/MT/SFT/RL pipeline around the compressed backbone to study how compression interacts with the native architecture, with compression applied directly at inference. As Figure~\ref{fig:teaser}(d) illustrates, progressive compression concentrates retained tokens on semantically critical regions while preserving spatial coverage. Experiments (\S\ref{sec:experiments}) show that native MLLMs tolerate far more redundancy than ViT-based ones: on our full-recipe Ours-2B backbone \method improves performance at 25\%--50\% token retention, and near-losslessness generalizes to other native backbones, all with about $2\times$ FLOPs reduction.

\begin{figure}[t]
  \centering
  \includegraphics[width=\linewidth]{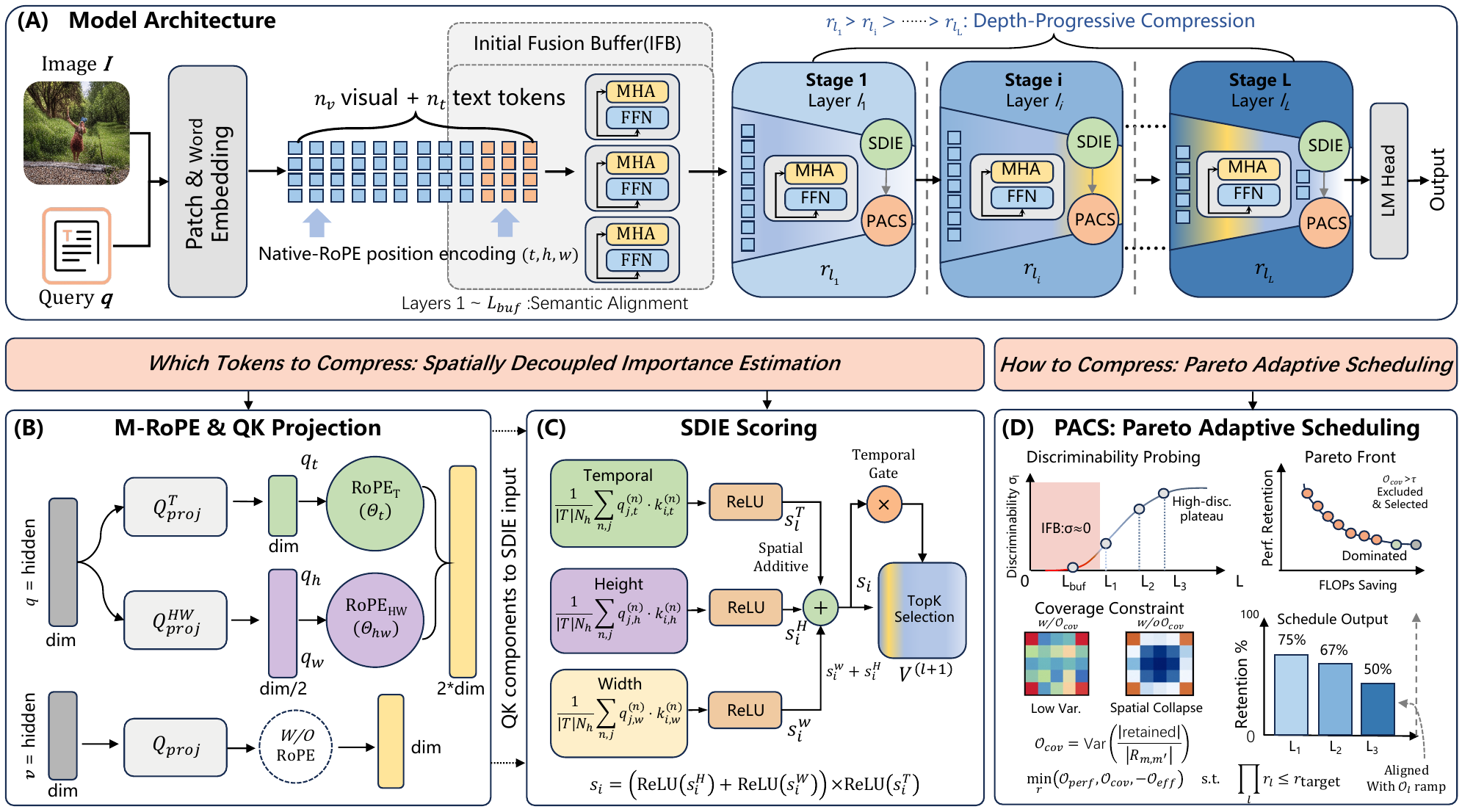}
\caption{\textbf{Overview of \method.}
\textbf{(A)}~Model architecture with depth-progressive compression ($r_{l_{1}}> r_{l_{2}}> \dots  > r_{l_{L}}$) applied after the IFB.
\textbf{(B)}~Native-RoPE QK projection: queries / keys are split into temporal ($\boldsymbol{q}_t$) and spatial ($\boldsymbol{q}_h, \boldsymbol{q}_w$) components with independent rotary encodings.
\textbf{(C)}~SDIE scoring: directional saliency scores are fused via spatial addition and temporal multiplicative gating for TopK selection.
\textbf{(D)}~PACS: discriminability probing identifies compression layers; a Pareto front over performance, efficiency, and coverage yields the optimal retention schedule.}
\label{fig:framework}
\end{figure}

\section{Method}
\label{sec:method}

\subsection{Framework Overview}
\label{sec:overview}

As illustrated in Figure~\ref{fig:framework}, \method operates on a native MLLM backbone and performs progressive visual token compression along the model depth. Given an input image $\boldsymbol{I}$ and a text query $\boldsymbol{q}$, the native backbone first encodes them into visual tokens $\mathcal{V}^{(0)} = \{v_i\}_{i=1}^{n_v}$, where each $v_i \in \mathbb{R}^d$, and text tokens $\mathcal{T} = \{t_j\}_{j=1}^{n_t}$, where each $t_j \in \mathbb{R}^d$, via the Patch and Word Embedding layer, respectively. 
The mixed sequence then passes through the first $L_{\mathrm{buf}}$ layers, which constitute the \textit{Initial Fusion Buffer (IFB)} for cross-modal semantic alignment and position registration, where $L_{\mathrm{buf}}$ is identified from the probing analysis as the layer beyond which discriminability begins to rise steadily (details in \S\ref{sec:impl_details}).
After the IFB, compression is executed at a predefined set of stages $\mathcal{S} = \{l_1, l_2, l_3\}$ with $L_{\mathrm{buf}} < l_1 < l_2 < l_3 \leq L$:
\begin{equation}
\mathcal{V}^{(l+1)} = \mathrm{TopK}\!\bigl(\mathcal{V}^{(l)},\; s^{(l)},\; K^{(l)}\bigr), \quad l \in \mathcal{S},
\label{eq:overall}
\end{equation}
where $\mathrm{TopK}$ selects the $K^{(l)}$ tokens with the highest importance scores, $\mathcal{V}^{(l)}$ denotes the visual token set at layer $l$, $s^{(l)} = {s_i}{i \in \mathcal{V}^{(l)}}$ is the spatially decoupled importance score computed by SDIE (\S\ref{sec:sdie}), and $K^{(l)} = r_l \cdot |\mathcal{V}^{(l)}|$ is the number of tokens to retain at that stage, with the retention ratio $r_l$ determined adaptively by PACS (\S\ref{sec:pacs}). The progressive compression structure is deliberate: \textbf{compression intensity deepens gradually layer by layer} ($r{l_1} > r_{l_2} > r_{l_3}$), aligned with the depth-progressive discriminability ramp, where early post-IFB layers warrant conservative pruning and deeper layers on the high-discriminability plateau admit more aggressive compression.

\subsection{Spatially Decoupled Importance Estimation}
\label{sec:sdie}

Our probing analysis reveals that redundancy distributes independently and heterogeneously along the height and width dimensions: a token may be critical along only one spatial axis yet redundant along the other.  
Native-RoPE provides a natural structural prior for this decoupling. In native attention, the Query and Key vectors are partitioned along the head dimension into three components encoding temporal ($t$), height ($h$), and width ($w$):
\begin{equation}
\boldsymbol{q}_i = [\boldsymbol{q}_{i,t};\; \boldsymbol{q}_{i,h};\; \boldsymbol{q}_{i,w}] \in \mathbb{R}^{d_h}, \quad \boldsymbol{k}_j = [\boldsymbol{k}_{j,t};\; \boldsymbol{k}_{j,h};\; \boldsymbol{k}_{j,w}] \in \mathbb{R}^{d_h},
\label{eq:qk_decomp}
\end{equation}
where $d_h$ denotes the total QK head dimension, and each component is equipped with its own rotary positional encoding $\Theta_T = \{\beta_T^{-2k/d}\}_{k \in [0, d/2)}$, $\Theta_H = \{\beta_H^{-4i/d}\}_{i \in [0, d/4)}$, $\Theta_W = \{\beta_W^{-4j/d}\}_{j \in [0, d/4)}$ with distinct base frequencies $\beta_T, \beta_H, \beta_W$. The attention score naturally decomposes into three additive terms:
\begin{equation}
\mathrm{attn}(i,j) = \underbrace{\boldsymbol{q}_{i,t} \!\cdot\! \boldsymbol{k}_{j,t}^{\top}}_{\text{temporal term}} + \underbrace{\boldsymbol{q}_{i,h} \!\cdot\! \boldsymbol{k}_{j,h}^{\top}}_{\text{height term}} + \underbrace{\boldsymbol{q}_{i,w} \!\cdot\! \boldsymbol{k}_{j,w}^{\top}}_{\text{width term}}.
\label{eq:attn_decomp}
\end{equation}
We aggregate text-side attention from each component across all attention heads to compute directional importance scores for each visual token $i \in \mathcal{V}^{(l)}$:
\begin{equation}
\small
s_i^{T} = \frac{1}{|\mathcal{T}| \cdot N_h}\sum_{n=1}^{N_h}\sum_{j \in \mathcal{T}} \boldsymbol{q}_{j,t}^{(n)} \!\cdot\! {\boldsymbol{k}_{i,t}^{(n)}}^{\top},
s_i^{H} = \frac{1}{|\mathcal{T}| \cdot N_h}\sum_{n=1}^{N_h}\sum_{j \in \mathcal{T}} \boldsymbol{q}_{j,h}^{(n)} \!\cdot\! {\boldsymbol{k}_{i,h}^{(n)}}^{\top},
s_i^{W} = \frac{1}{|\mathcal{T}| \cdot N_h}\sum_{n=1}^{N_h}\sum_{j \in \mathcal{T}} \boldsymbol{q}_{j,w}^{(n)} \!\cdot\! {\boldsymbol{k}_{i,w}^{(n)}}^{\top}
\label{eq:directional_scores}
\end{equation}
where $N_h$ is the number of attention heads, the superscript $(n)$ indexes the $n$-th head, and $s_i^T, s_i^H, s_i^W \in \mathbb{R}$ denote the temporal, height directional, and width directional saliency of visual token $i$, respectively. Since these scores are raw dot products (not softmaxed) and may be negative, we apply ReLU to ensure non-negativity before fusion:
\begin{equation}
s_i = \bigl(\text{ReLU}(s_i^{H}) + \text{ReLU}(s_i^{W})\bigr) \times \text{ReLU}(s_i^{T}).
\label{eq:sdie_fusion}
\end{equation}

The additive spatial aggregation preserves tokens salient in \emph{either} the H or W dimension, consistent with the cross-dimensional heterogeneity observed in the probing analysis, while the multiplicative temporal gate suppresses globally irrelevant tokens ($s_i^T \approx 0$). SDIE scores are extracted from intermediate QK products already computed in the native forward pass, introducing zero additional parameters. A formal optimality proof for this fusion form is given in Appendix~\ref{thm:sdie_optimal}.

\subsection{Pareto Adaptive Compression Scheduling}
\label{sec:pacs}

Since visual token redundancy accumulates progressively with depth and the optimal compression ratio varies across stages, we formulate cross-layer budget allocation to reach the Pareto frontier of performance, efficiency, and spatial coverage trade-offs.

\begin{figure*}[t]
  \centering
  \includegraphics[width=\linewidth]{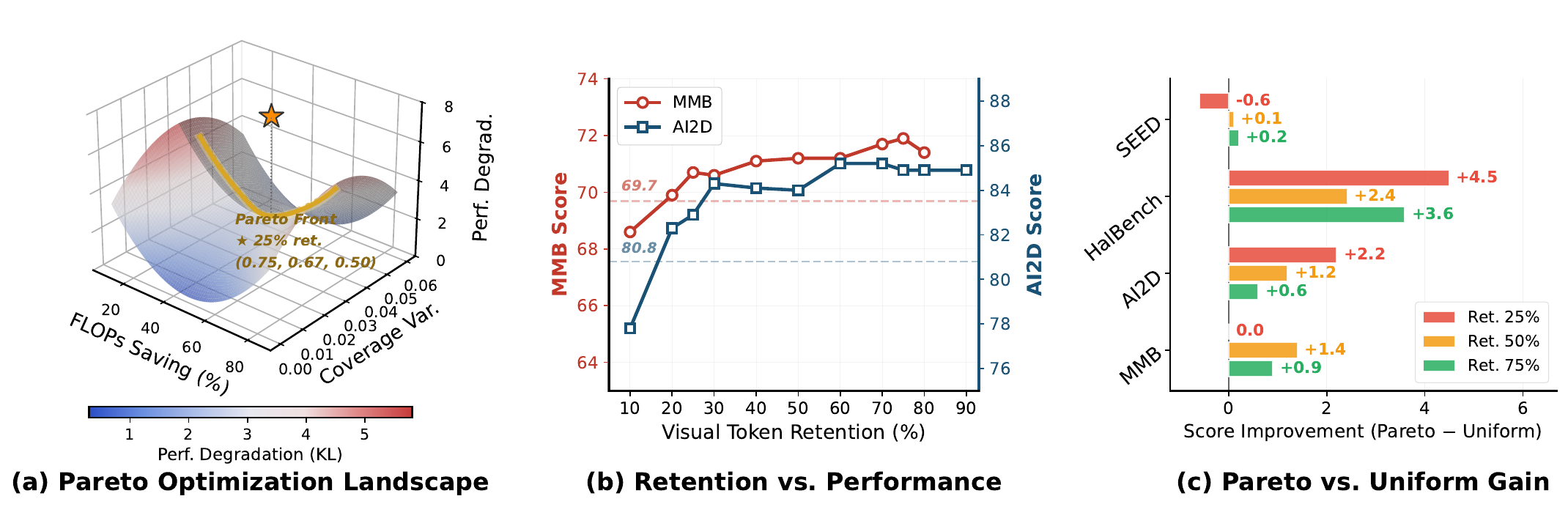}
  \caption{\textbf{Pareto front visualization for compression schedule optimization.} \textbf{(a)}~3D Pareto front over performance preservation, FLOPs saving, and coverage variance; the gold star marks the selected operating point. \textbf{(b)}~2D projection of FLOPs saving vs.\ performance retention; color encodes coverage variance (green=low, red=high). \textbf{(c)}~Per-layer optimal retention rate: three compression stages (orange/dark orange/red-orange) are aligned with the discriminability curve.}
  \label{fig:pareto}
\end{figure*}

\noindent
\textbf{Problem formulation.}
Let the compression stages be $\mathcal{S} = \{l_1, l_2, l_3\}$ with per-stage retention ratios $\boldsymbol{r} = (r_{l_1}, r_{l_2}, r_{l_3}) \in [0,1]^3$ as decision variables. The optimization problem is defined as:
\begin{equation}
\min_{\boldsymbol{r}} \;\bigl(\mathcal{O}_{\text{perf}}(\boldsymbol{r}),\;\; \mathcal{O}_{\text{cov}}(\boldsymbol{r}),\;\; -\mathcal{O}_{\text{eff}}(\boldsymbol{r})\bigr), \quad \text{s.t.} \;\prod_{l \in \mathcal{S}} r_l \leq r_{\text{target}},
\label{eq:pareto_opt}
\end{equation}
where $\mathcal{O}_{\text{perf}}(\boldsymbol{r}) = \frac{1}{|\mathcal{B}|}\sum_{b \in \mathcal{B}} \bigl(S_b^{\text{full}} - S_b^{\text{comp}}(\boldsymbol{r})\bigr)$ denotes the average performance degradation across a benchmark suite $\mathcal{B}$ evaluated on a held-out calibration set, $\mathcal{O}_{\text{eff}}(\boldsymbol{r}) = 1 - \prod_{l \in \mathcal{S}} r_l$ is the total reduction ratio, $\mathcal{O}_{\text{cov}}$ is the spatial coverage variance (Eq.~\ref{eq:coverage}), and $r_{\text{target}}$ is the target overall retention upper bound. All three objectives are min-max normalized to $[0, 1]$ on the calibration set to ensure comparable scales during Pareto front construction. The Pareto optimal solution set is:
\begin{equation}
\mathcal{P}^* = \Bigl\{\boldsymbol{r} \;\Big|\; \prod_{l \in \mathcal{S}} r_l \leq r_{\text{target}},\; \nexists\, \boldsymbol{r}' \text{ s.t. } \boldsymbol{r}' \prec \boldsymbol{r}\Bigr\}.
\label{eq:pareto_set}
\end{equation}

\noindent
\textbf{Spatial coverage constraint.}
Aggressive token dropping can concentrate retained tokens in salient regions while leaving large spatial gaps, degrading performance on tasks requiring global spatial understanding (\eg, scene description, spatial grounding). To address this, the image is partitioned into $M_h \times M_w$ sub-regions $\{R_{m_h,m_w}\}$, where each $R_{m_h,m_w} \subset \mathcal{V}$ contains the visual tokens whose spatial coordinates $(h_i, w_i)$ fall within the $m_h$-th row and $m_w$-th column of the grid. We set $M_h = M_w = 4$ to balance spatial granularity with the number of tokens per sub-region (sensitivity analysis in \S\ref{sec:impl_details}). Coverage variance is defined as:
\begin{equation}
\mathcal{O}_{\text{cov}}(\boldsymbol{r}) = \text{Var}_{(m_h,m_w)}\left(\frac{|\{i \in R_{m_h,m_w} : x_i = 1\}|}{|R_{m_h,m_w}|}\right),
\label{eq:coverage}
\end{equation}
where $x_i \in \{0,1\}$ denotes the retention decision for token $v_i$. Candidate points with $\mathcal{O}_{\mathrm{cov}}$ exceeding a threshold $\tau_{\mathrm{cov}}$ are automatically excluded during traversal, making spatial coverage an \textbf{intrinsic constraint} of the Pareto selection rather than a post-hoc correction. Our experiments confirm that incorporating spatial coverage as an explicit objective yields measurable gains on tasks requiring holistic scene understanding, validating its inclusion in the multi objective formulation.

\noindent
\textbf{Front traversal and operating point selection.}
We enumerate $\boldsymbol{r} \in \mathcal{R}_\delta^{|\mathcal{S}|}$ on a discrete grid $\mathcal{R}_\delta = \{1.0,\; 1.0{-}\delta,\; \ldots,\; r_{\min}\}$, evaluate the three objectives on a calibration set, filter the Pareto set $\mathcal{P}^*$, and select the operating point subject to the $r_{\text{target}}$ constraint. The discriminability curve from the probing analysis determines the front shape: stages aligned with the discriminability ramp correspond to moderate pruning, while those on the high plateau admit the most aggressive compression (specific configurations in \S\ref{sec:impl_details}). The resulting Pareto front is illustrated in Figure~\ref{fig:pareto}.

\subsection{Training Strategy}
\label{sec:training}
Since no concurrent Native MLLM fully open-sources all training stages (PT / MT / SFT / RL), to further investigate the compression behavior of native MLLM across all training stages, we train a 2B-parameter native MLLM \emph{from scratch}.
Specifically,
\method follows a four-stage 
training pipeline:
\textit{i)} \textbf{Pre-training} ($\sim$518M samples) learns basic visual concepts,
and the Post-LLM backbone remains frozen.
\textit{ii)} \textbf{Mid-training} ($\sim$89M samples) and \textit{iii)} \textbf{SFT} ($\sim$12.8M samples) strengthens vision-language understanding and alignment by using high-quality multi-source training data.
All parameters are optimized end-to-end.
\textit{iv)} \textbf{RL} ($\sim$12K samples) applies GRPO~\citep{shao2024deepseekmath} with rule-based reward and category-proportional priority sampling for hard examples.
We provide the simplified schematic illustration of data distribution in Figure~\ref{fig:placeholder} and more details in Appendix.

\begin{figure}[t]
    \centering
    \includegraphics[width=1\linewidth]{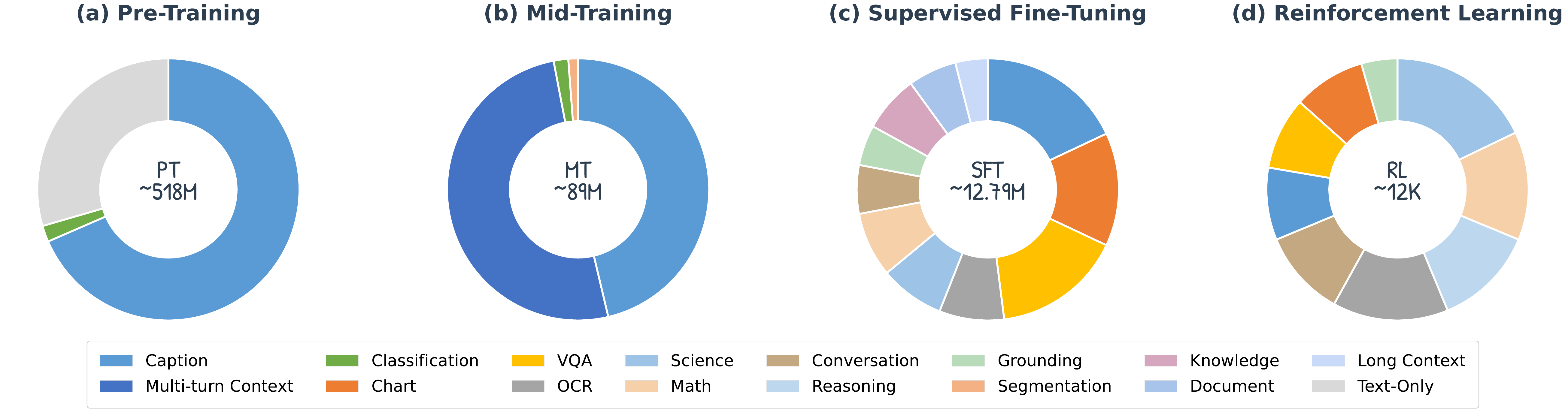}
  \caption{\textbf{Training data distribution across four stages.} Each donut chart shows the dataset composition and proportional contribution per stage.}
    \label{fig:placeholder}
\end{figure}

\begin{table}[t]
\centering
\footnotesize
\setlength{\tabcolsep}{3.5pt}
\label{tab:main_methods}
\caption{\textbf{Compression robustness across visual token retention ratios.} Each cell reports the relative change from the uncompressed baseline (absolute score in parentheses); \pos{green}/\negitive{red} denote \pos{improvement}/\negitive{degradation}. $^\dagger$ViT-free MLLMs. $^\star$MME: sum of perception and reasoning. $^\ddagger$HallusionBench: average of aAcc, qAcc, and fAcc as commonly done in prior works.}
\label{tab:main_results}
\resizebox{\textwidth}{!}{%
\begin{tabular}{llcccccccc}
\toprule
 & & \textbf{MMB} & \textbf{AI2D} & \textbf{MME}$^\star$ & \textbf{HallBench}$^\ddagger$ & \textbf{RWQA} & \textbf{MMStar} & \textbf{SEED-IMG} & \textbf{MathVista} \\
\midrule
\rowcolor{gray!8}
\multicolumn{10}{c}{\textbf{Retention=25\%}} \\
\cmidrule(lr){1-10}
\multirow{2}{*}{Ours-2B$^\dagger$} & \cellcolor{blue!8}\textit{w/o compress} & \cellcolor{blue!8}69.7 & \cellcolor{blue!8}80.8 & \cellcolor{blue!8}2084 & \cellcolor{blue!8}42.7 & \cellcolor{blue!8}61.8 & \cellcolor{blue!8}52.9 & \cellcolor{blue!8}71.8 & \cellcolor{blue!8}55.7 \\
 & \cellcolor{blue!8}\textit{w/ compress} & \cellcolor{blue!8}\pos{+1.4\%} \real{(70.7)} & \cellcolor{blue!8}\pos{+2.8\%} \real{(83.1)} & \cellcolor{blue!8}\pos{+0.4\%} \real{(2092)} & \cellcolor{blue!8}\negitive{-2.5\%} \real{(41.6)} & \cellcolor{blue!8}\negitive{-1.0\%} \real{(61.2)} & \cellcolor{blue!8}\negitive{-4.7\%} \real{(50.4)} & \cellcolor{blue!8}\negitive{-0.6\%} \real{(71.4)} & \cellcolor{blue!8}\negitive{-9.7\%} \real{(50.3)} \\
\midrule
\multirow{2}{*}{SAIL-7B$^\dagger$} & \textit{w/o compress} & 70.1 & 76.7 & 1719 & 33.1 & 63.9 & 53.1 & 72.9 & 57.0 \\
 & \textit{w/ compress} & \negitive{-2.6\%} \real{(68.3)} & \negitive{-2.2\%} \real{(75.0)} & \negitive{-2.5\%} \real{(1676)} & \negitive{-1.0\%} \real{(32.8)} & \negitive{-1.1\%} \real{(63.2)} & \negitive{-13.6\%} \real{(45.9)} & \negitive{-3.0\%} \real{(70.7)} & \negitive{-11.2\%} \real{(50.6)} \\
\midrule
\multirow{2}{*}{ONECAT-3B$^\dagger$} & \textit{w/o compress} & 83.1 & 77.2 & 2064 & 39.4 & 64.8 & 55.6 & 72.8 & 60.7 \\
 & \textit{w/ compress} & \negitive{-3.7\%} \real{(80.0)} & \negitive{-5.5\%} \real{(73.0)} & \negitive{-2.1\%} \real{(2021)} & \negitive{-3.5\%} \real{(38.0)} & \negitive{-3.1\%} \real{(62.8)} & \negitive{-3.8\%} \real{(53.5)} & \negitive{-4.1\%} \real{(69.8)} & \negitive{-18.3\%} \real{(49.6)} \\
\midrule
\multirow{2}{*}{Qwen3-VL} & \textit{w/o compress} & 78.4 & 76.9 & 2014.5 & 51.4 & 69.5 & 58.3 & 77.2 & 61.3 \\
 & \textit{w/ compress} & \negitive{-11.9\%} \real{(69.1)} & \negitive{-8.7\%} \real{(70.2)} & \negitive{-6.8\%} \real{(1877)} & \negitive{-30.5\%} \real{(35.7)} & \negitive{-8.1\%} \real{(63.9)} & \negitive{-23.2\%} \real{(44.8)} & \negitive{-4.7\%} \real{(73.6)} & \negitive{-43.7\%} \real{(34.5)} \\
\midrule
\multirow{2}{*}{Intern3-VL} & \textit{w/o compress} & 81.1 & 78.7 & 2221 & 42.5 & 66.3 & 60.7 & 76.5 & 57 \\
 & \textit{w/ compress} & \negitive{-6.3\%} \real{(76.0)} & \negitive{-7.1\%} \real{(73.1)} & \negitive{-5.1\%} \real{(2107)} & \negitive{-4.9\%} \real{(40.4)} & \negitive{-3.8\%} \real{(63.8)} & \negitive{-11.5\%} \real{(53.7)} & \negitive{-8.0\%} \real{(70.4)} & \negitive{-26.5\%} \real{(41.9)} \\
\midrule
\rowcolor{gray!8}
\multicolumn{10}{c}{\textbf{Retention=50\%}} \\
\cmidrule(lr){1-10}
\multirow{2}{*}{Ours-2B$^\dagger$} & \cellcolor{blue!8}\textit{w/o compress} & \cellcolor{blue!8}69.7 & \cellcolor{blue!8}80.8 & \cellcolor{blue!8}2084 & \cellcolor{blue!8}42.7 & \cellcolor{blue!8}61.8 & \cellcolor{blue!8}52.9 & \cellcolor{blue!8}71.8 & \cellcolor{blue!8}55.7 \\
 & \cellcolor{blue!8}\textit{w/ compress} & \cellcolor{blue!8}\pos{+2.2\%} \real{(71.2)} & \cellcolor{blue!8}\pos{+4.0\%} \real{(84.0)} & \cellcolor{blue!8}\pos{+0.6\%} \real{(2098)} & \cellcolor{blue!8}\pos{+0.5\%} \real{(42.9)} & \cellcolor{blue!8}\pos{+1.3\%} \real{(62.6)} & \cellcolor{blue!8}\negitive{-1.7\%} \real{(52.0)} & \cellcolor{blue!8}\pos{+1.0\%} \real{(72.5)} & \cellcolor{blue!8}\negitive{-2.5\%} \real{(54.3)} \\
\midrule
\multirow{2}{*}{SAIL-7B$^\dagger$} & \textit{w/o compress} & 70.1 & 76.7 & 1719 & 33.1 & 63.9 & 53.1 & 72.9 & 57.0 \\
 & \textit{w/ compress} & \negitive{-2.3\%} \real{(68.5)} & 0.0\% \real{(76.7)} & \negitive{-0.4\%} \real{(1712)} & \pos{+1.0\%} \real{(33.5)} & \negitive{-0.2\%} \real{(63.8)} & \negitive{-10.4\%} \real{(47.6)} & \negitive{-1.7\%} \real{(71.7)} & \negitive{-5.3\%} \real{(54.0)} \\
\midrule
\multirow{2}{*}{ONECAT-3B$^\dagger$} & \textit{w/o compress} & 83.1 & 77.2 & 2064 & 39.4 & 64.8 & 55.6 & 72.8 & 60.7 \\
 & \textit{w/ compress} & \negitive{-3.7\%} \real{(80.0)} & \negitive{-2.6\%} \real{(75.2)} & \negitive{-1.9\%} \real{(2025)} & \negitive{-2.2\%} \real{(38.5)} & \negitive{-2.8\%} \real{(63.0)} & \pos{+0.9\%} \real{(56.1)} & \negitive{-2.1\%} \real{(71.3)} & \negitive{-19.7\%} \real{(48.8)} \\
\midrule
\multirow{2}{*}{Qwen3-VL} & \textit{w/o compress} & 78.4 & 76.9 & 2015 & 51.4 & 69.5 & 58.3 & 77.2 & 61.3 \\
 & \textit{w/ compress} & \negitive{-4.0\%} \real{(75.3)} & \negitive{-7.7\%} \real{(71.0)} & \negitive{-4.4\%} \real{(1925)} & \negitive{-29.3\%} \real{(36.3)} & \negitive{-4.9\%} \real{(66.1)} & \negitive{-21.3\%} \real{(45.9)} & \negitive{-4.3\%} \real{(73.9)} & \negitive{-41.4\%} \real{(35.9)} \\
\midrule
\multirow{2}{*}{Intern3-VL} & \textit{w/o compress} & 81.1 & 78.7 & 2221.2 & 42.5 & 66.3 & 60.7 & 76.5 & 57 \\
 & \textit{w/ compress} & \negitive{-6.3\%} \real{(76.0)} & \negitive{-6.1\%} \real{(73.9)} & \negitive{-6.0\%} \real{(2087)} & \negitive{-2.6\%} \real{(41.4)} & \negitive{-2.6\%} \real{(64.6)} & \negitive{-11.9\%} \real{(53.5)} & \negitive{-2.9\%} \real{(74.3)} & \negitive{-23.9\%} \real{(43.4)} \\
\midrule
\rowcolor{gray!8}
\multicolumn{10}{c}{\textbf{Retention=75\%}} \\
\cmidrule(lr){1-10}
\multirow{2}{*}{Ours-2B$^\dagger$} & \cellcolor{blue!8}\textit{w/o compress} & \cellcolor{blue!8}69.7 & \cellcolor{blue!8}80.8 & \cellcolor{blue!8}2084 & \cellcolor{blue!8}42.7 & \cellcolor{blue!8}61.8 & \cellcolor{blue!8}52.9 & \cellcolor{blue!8}71.8 & \cellcolor{blue!8}55.7 \\
 & \cellcolor{blue!8}\textit{w/ compress} & \cellcolor{blue!8}\pos{+3.2\%} \real{(71.9)} & \cellcolor{blue!8}\pos{+5.1\%} \real{(84.9)} & \cellcolor{blue!8}\pos{+1.7\%} \real{(2120.2)} & \cellcolor{blue!8}\pos{+3.9\%} \real{(44.4)} & \cellcolor{blue!8}\pos{+1.0\%} \real{(62.4)} & \cellcolor{blue!8}\negitive{-0.6\%} \real{(52.6)} & \cellcolor{blue!8}\pos{+1.8\%} \real{(73.1)} & \cellcolor{blue!8}\pos{+1.3\%} \real{(56.4)} \\
\midrule
\multirow{2}{*}{SAIL-7B$^\dagger$} & \textit{w/o compress} & 70.1 & 76.7 & 1719 & 33.1 & 63.9 & 53.1 & 72.9 & 57.0 \\
 & \textit{w/ compress} & \negitive{-1.6\%} \real{(69.0)} & \negitive{-0.7\%} \real{(76.2)} & \negitive{-1.4\%} \real{(1695)} & \pos{+1.6\%} \real{(33.7)} & \pos{+1.1\%} \real{(64.6)} & 0.0\% \real{(53.1)} & \negitive{-0.4\%} \real{(72.6)} & \negitive{-1.7\%} \real{(56.1)} \\
\midrule
\multirow{2}{*}{ONECAT-3B$^\dagger$} & \textit{w/o compress} & 83.1 & 77.2 & 2063.9 & 39.4 & 64.8 & 55.6 & 72.8 & 60.7 \\
 & \textit{w/ compress} & \negitive{-4.1\%} \real{(80.6)} & \negitive{-2.0\%} \real{(75.7)} & \negitive{-1.0\%} \real{(2043)} & \negitive{-5.5\%} \real{(37.2)} & \negitive{-0.3\%} \real{(64.6)} & 0.0\% \real{(55.6)} & \negitive{-1.8\%} \real{(71.5)} & \negitive{-26.0\%} \real{(44.9)} \\
\midrule
\multirow{2}{*}{Qwen3-VL} & \textit{w/o compress} & 78.4 & 76.9 & 2015 & 51.4 & 69.5 & 58.3 & 77.2 & 61.3 \\
 & \textit{w/ compress} & \negitive{-4.7\%} \real{(74.7)} & \negitive{-7.5\%} \real{(71.1)} & \negitive{-3.6\%} \real{(1942)} & \negitive{-26.2\%} \real{(37.9)} & \negitive{-5.8\%} \real{(65.5)} & \negitive{-21.3\%} \real{(45.9)} & \negitive{-4.0\%} \real{(74.1)} & \negitive{-35.1\%} \real{(39.8)} \\
\midrule
\multirow{2}{*}{Intern3-VL} & \textit{w/o compress} & 81.1 & 78.7 & 2221 & 42.5 & 66.3 & 60.7 & 76.5 & 57.0 \\
 & \textit{w/ compress} & \negitive{-3.0\%} \real{(78.7)} & \negitive{-6.1\%} \real{(73.9)} & \negitive{-6.0\%} \real{(2089)} & \negitive{-6.2\%} \real{(39.8)} & \negitive{-3.3\%} \real{(64.1)} & \negitive{-11.2\%} \real{(53.9)} & \negitive{-1.5\%} \real{(75.3)} & \negitive{-20.5\%} \real{(45.3)} \\
\bottomrule
\end{tabular}%
}
\end{table}

\vspace{-5pt}

\section{Experiments}
\label{sec:experiments}

\subsection{Experimental Setup}
\label{sec:impl_details}

\noindent
\textbf{Implementation details.}
\method is built upon a ViT-free native multimodal large language model based on the Qwen3-1.7B~\citep{yang2025qwen3} backbone.
The depth of IFB is $L_{\text{buf}}=4$.
The native Native-RoPE base frequencies are $\beta_T = 1 \times 10^6$ and $\beta_H = \beta_W = 1 \times 10^4$. 
Full architecture configurations are in Table~\ref{tab:arch_details} (Appendix).
All training stages use 5$\times$8 H20-96G GPUs with bfloat16 precision. 
More training details are in Table~\ref{tab:hyperparam} (Appendix).

\noindent
\textbf{Evaluation methods.}
In Table~\ref{tab:main_results}, we apply \method to four representative MLLMs spanning both paradigms, \textit{i.e.}, the ViT-free native models SAIL-7B~\citep{sail25} and ONECAT-3B~\citep{li2025onecat}, and the ViT-based modular models Qwen3-VL~\citep{bai2025qwen3} and InternVL3~\citep{zhu2025internvl3}.
In Table~\ref{tab:compression_methods}, we instead transfer existing modular training-free compression methods to our own native MLLM backbone.
All comparison experiments are implemented based on the official checkpoints or codebases released by the official authors.

\noindent \textbf{Front traversal and operating point selection.} The search jointly optimizes layer positions and per-stage retention ratios. Three ordered stages $\mathcal{S}=\{l_1,l_2,l_3\}$ ($l_1{<}l_2{<}l_3$) are chosen from the 36 post-IFB layers,
giving $\binom{36}{3}=7{,}140$ configurations. Each layer-wise retention rate $r_i$ is drawn from a shared 25-value set $\mathcal{G}=\{0.10,0.15,\ldots,0.90\}\cup\{k/9\}_{k=1}^{8}$, yielding $25^3=15{,}625$ ratio combinations and $\sim\!1.12{\times}10^{8}$ total candidates. We evaluate the three objectives on a calibration set, extract $\mathcal{P}^{*}$ via dominance filtering, and pick the point satisfying $r_{\text{target}}$. 
Guided by \textbf{Finding~1} (depth-progressive redundancy accumulation), the monotonicity relation $r_1 > r_2 > r_3$ naturally emerges, substantially narrowing the traversal space. 
the resulting front is shown in Figure~\ref{fig:pareto}.


\noindent
\textbf{Evaluation benchmarks.} We evaluate with VLMEvalKit~\citep{duan2024vlmevalkit} on eight benchmarks, \textit{i.e.,} MMBench-DEV-EN (MMB)~\citep{liu2024mmbench}, AI2D~\citep{kembhavi2016diagram}, MME~\citep{fu2023mme}, HallusionBench (HallBench)~\citep{guan2024hallusionbench}, RealWorldQA (RWQA), MMStar~\citep{chen2024we}, SEEDBench-IMG (SEED)~\citep{li2023seed}, and MathVista~\citep{lu2023mathvista}.

 \begin{table}[t]
\centering
\caption{\textbf{Comparison with transferred modular training-free compression methods on native MLLM.} $\dagger$ denotes transfer from modular architectures. Best results per retention level are \textbf{bolded}.}
\label{tab:compression_methods}
\begin{adjustbox}{max width=\linewidth}
\begin{tabular}{ll cccccccc}
\toprule
\textbf{Method} & \textbf{Ret.} & \textbf{MMB} & \textbf{AI2D} & \textbf{MME} & \textbf{HalBench} & \textbf{RWQA} & \textbf{MMStar} & \textbf{SEED} & \textbf{MathVista} \\
\midrule
\rowcolor{gray!10} Ours-2B (full) & 100\% & 69.7 & 80.8 & 2084 & 42.7 & 61.8 & 52.9 & 71.8 & 55.7 \\
\midrule
\multicolumn{10}{c}{\textit{25\% Retention}} \\
\midrule
FastV$^\dagger$ & 25\% & \negative{-27.7\%} \real{(50.4)} & \negative{-39.9\%} \real{(48.6)} & \negative{-31.3\%} \real{(1432)} & \negative{-27.9\%} \real{(30.8)} & \negative{-63.3\%} \real{(22.7)} & \negative{-26.1\%} \real{(39.1)} & \negative{-25.9\%} \real{(53.2)} & \negative{-32.7\%} \real{(37.5)} \\
PyramidDrop$^\dagger$ & 25\% & \negative{-3.2\%} \real{(67.5)} & \negative{-2.1\%} \real{(79.1)} & \positive{-4.4\%} \real{(1992)} & \negative{-7.8\%} \real{(39.4)} & \negative{-2.9\%} \real{(60.0)} & \negative{-12.3\%} \real{(46.4)} & \negative{-5.7\%} \real{(67.7)} & -\negative{\textbf{-13.5\%}} \real{(48.2)} \\
D$^2$Pruner$^\dagger$ & 25\% & \negative{-3.2\%} \real{(67.5)} & \negative{-11.9\%} \real{(71.2)} & \negative{-3.5\%} \real{(2011)} & \negative{-13.2\%} \real{(37.1)} & \negative{-9.1\%} \real{(56.2)} & \negative{-15.9\%} \real{(44.5)} & \negative{-7.2\%} \real{(66.6)} & \negative{-18.1\%} \real{(45.6)} \\
SparseVLM$^\dagger$ & 25\% & \negative{-1.1\%} \real{(68.9)} & \positive{\textbf{+3.7\%}} \real{(83.8)} & \positive{-1.2\%} \real{(2059)} & \negative{-4.5\%} \real{(40.8)} & \negative{-2.1\%} \real{(60.5)} & \negative{-9.8\%} \real{(47.7)} & \negative{-3.8\%} \real{(69.1)} & \negative{\textbf{-6.5\%}} \real{(52.1)} \\
\rowcolor{blue!8}
\textbf{Ours} & 25\% & \positive{\textbf{+1.4\%}} \real{(70.7)} & \positive{+2.8\%} \real{(83.1)} & \positive{+0.4\%} \real{(2092)}  & \negative{\textbf{-2.6\%}} \real{(41.6)} & \negative{\textbf{-1.0\%}} \real{(61.2)} & \negative{\textbf{-4.7\%}} \real{(50.4)} & \negative{\textbf{-0.6\%}} \real{(71.4)} & \negative{-9.7\%} \real{(50.3)} \\
\midrule
\multicolumn{10}{c}{\textit{50\% Retention}} \\
\midrule
FastV$^\dagger$ & 50\% & \negative{-7.3\%} \real{(64.6)} & \negative{-24.8\%} \real{(60.8)} & \negative{-12.6\%} \real{(1821)} & \negative{-10.6\%} \real{(38.2)} & \negative{-50.2\%} \real{(30.8)} & \negative{-12.9\%} \real{(46.1)} & \negative{-8.6\%} \real{(65.6)} & \negative{-16.2\%} \real{(46.7)} \\
PyramidDrop$^\dagger$ & 50\% & \positive{+1.6\%} \real{(70.8)} & \positive{+1.5\%} \real{(82)} & \positive{+0.4\%} \real{(2092)} & \negative{-3.1\%} \real{(41.4)} & \negative{-2.4\%} \real{(60.3)} & \negative{-3.6\%} \real{(51.0)} & \negative{-0.4\%} \real{(71.5)} & \negative{-3.2\%} \real{(53.9)} \\
D$^2$Pruner$^\dagger$ & 50\% & \negative{-1.1\%} \real{(68.9)} & \positive{+2.0\%} \real{(82.4)} & \positive{\textbf{+0.5\%}} \real{(2096)} & \negative{-2.6\%} \real{(41.6)} & \negative{-2.9\%} \real{(60)} & \negative{-5.3\%} \real{(50.1)} & \negative{-0.4\%} \real{(71.5)} & \negative{-6.6\%} \real{(52.0)} \\
SparseVLM$^\dagger$ & 50\% & \positive{+1.6\%} \real{(70.8)} & \positive{\textbf{+5.2\%}} \real{(85)} & \negative{\textbf{-0.2\%}} \real{(2081)} & \positive{\textbf{+0.3\%}} \real{(42.8)} & \negative{\textbf{-4.6\%}} \real{(61.8)} & \negative{-3.4\%} \real{(51.1)} & \positive{\textbf{+0.1\%}} \real{(71.9)} & \negative{-1.1\%} \real{(55.1)} \\
\rowcolor{blue!8}
\textbf{Ours} & 50\% & \positive{\textbf{+2.2\%}} \real{(71.2)} & \positive{+4.0\%} \real{(84)} & \positive{0.6\%} \real{(2098)}  & \positive{\textbf{+0.4\%}} \real{(42.9)} & \positive{\textbf{+1.3\%}} \real{(62.6)} & \negative{\textbf{-1.7\%}} \real{(52)} & \positive{\textbf{+1.0\%}} \real{(72.5)} & \negative{\textbf{-2.5\%}} \real{(54.3)} \\
\midrule
\multicolumn{10}{c}{\textit{75\% Retention}} \\
\midrule
FastV$^\dagger$ & 75\% & \negative{-1.6\%} \real{(68.6)} & \negative{-15.7\%} \real{(68.1)} & \negative{-3.1\%} \real{(2021)} & \negative{-0.3\%} \real{(42.6)} & \negative{-43.4\%} \real{(35.0)} & \negative{-36.1\%} \real{(33.8)} & \negative{-0.8\%} \real{(71.2)} & \negative{-5.6\%} \real{(52.6)} \\
PyramidDrop$^\dagger$ & 75\% & \positive{+1.9\%} \real{(71)} & \positive{+3.6\%} \real{(83.7)} & \negative{-0.3\%} \real{(2077)} & \positive{+0.4\%} \real{(42.9)} & \negative{-0.6\%} \real{(61.4)} & \negative{-0.4\%} \real{(52.7)} & \positive{+1.4\%} \real{(72.8)} & \negative{-1.6\%} \real{(54.8)} \\
D$^2$Pruner$^\dagger$ & 75\% & \positive{+1.7\%} \real{(70.9)} & \positive{+3.8\%} \real{(83.9)} & \positive{+0.9\%} \real{(2103)} & \positive{+0.2\%} \real{(42.8)} & \negative{-1.8\%} \real{(60.7)} & \negative{0.0\%} \real{(52.9)} & \positive{+1.0\%} \real{(72.5)} & \negative{-3.8\%} \real{(53.6)} \\
SparseVLM$^\dagger$ & 75\% & \positive{+1.3\%} \real{(70.6)} & \positive{\textbf{+5.3\%}} \real{(85.1)} & \positive{+0.9\%} \real{(2104)} & \negative{-0.5\%} \real{(42.5)} & \positive{+0.6\%} \real{(62.2)} & \negative{-1.0\%} \real{(52.4)} & \positive{+1.0\%} \real{(72.5)} & \positive{+0.7\%} \real{(56.1)} \\
\rowcolor{blue!8}
\textbf{Ours} & 75\% & \positive{\textbf{+3.2\%}} \real{(71.9)} & \positive{+5.1\%} \real{(84.9)} & \positive{\textbf{+1.7\%}} \real{(2120)} & \positive{\textbf{+3.9\%}} \real{(44.4)} & \positive{\textbf{+1.0\%}} \real{(62.4)} & \negative{-0.6\%} \real{(52.6)} & \positive{\textbf{+1.8\%}} \real{(73.1)} & \positive{\textbf{+1.3\%}} \real{(56.4)} \\
\bottomrule
\end{tabular}
\end{adjustbox}
\end{table}

\begin{figure}[t]
  \centering
  \includegraphics[width=\linewidth]{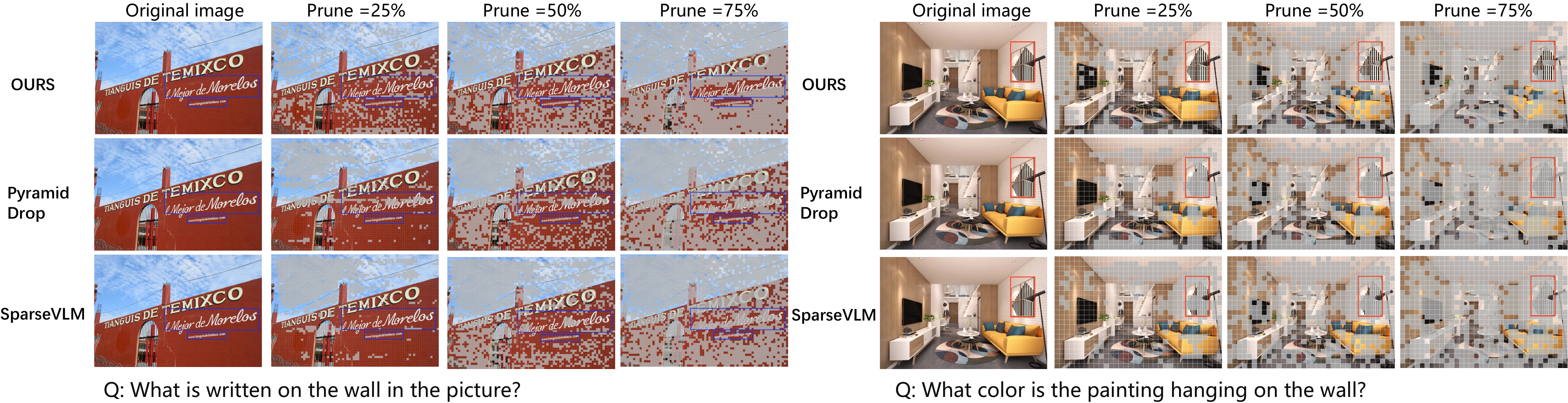}
\caption{\textbf{Progressive visual token retention maps at three compression stages.}
Retained tokens increasingly focus on semantically critical regions (e.g., text and objects) while preserving spatial coverage, validating the Pareto-adaptive schedule’s balance between performance and coverage.}
  \label{fig:exp1}
\end{figure}
\vspace{-5pt}
 
 
\subsection{Main Results}
\label{sec:main_results}

Table~\ref{tab:main_results} compares \method against modular MLLMs at 25\%, 50\%, and 75\% retention, reporting relative change from each model's uncompressed baseline. \method uniquely benefits from compression: it gains on MMB (+1.4\%$\to$+3.2\%) and AI2D (+2.8\%$\to$+5.1\%) as retention increases and stays near-lossless at 25\% on most benchmarks except MMStar and MathVista, whereas modular MLLMs degrade consistently, e.g.\ Qwen3-VL collapses up to 43.7\% on MathVista at 25\%, reflecting our core motivation that ViT-free MLLMs accumulate far more redundant tokens. Applying SDIE and PACS training-free to other ViT-free MLLMs corroborates this: SAIL-7B, sharing the decoupled spatial encoding, keeps a tolerable margin on general benchmarks (MMB, AI2D, MME, HallBench, RWQA, SEED), while ONECAT-3B, which folds spatial information into a single 1-D signal, drops more (AI2D $-5.5\%$ vs.\ $-2.2\%$ at 25\%). The ordering is not uniform, however: on MMStar SAIL-7B drops more ($-13.6\%$ vs.\ $-3.8\%$) and both degrade on MathVista ($-11.2\%$, $-18.3\%$), reflecting sensitivity to dense mathematical detail rather than a generic 1-D failure. This cross-benchmark inconsistency in spatial sensitivity motivates SDIE's directional decoupling.
\subsection{Comparison with Transfer Baselines}
\label{sec:compression_comparison}

Table~\ref{tab:compression_methods} compares \method with modular compression methods transferred to the native MLLM backbone across retention levels.
FastV~\citep{chen2024image} degrades severely under aggressive compression, with MMB dropping from 69.7 to 50.4 (27.7\%) at 25\% retention, due to single-stage pruning and collapsing the decoupled $\left(T,H,W\right)$ signals into one dimension.
PyramidDrop~\citep{xing2024pyramiddrop} addresses only the former via multi-stage compression and thus still trails \method.
D$^2$Pruner~\citep{zhang2026d2pruner} uses diversity-aware scalar ranking, while SparseVLM~\citep{zhang2024sparsevlm} re-weights tokens by text-visual relevance; both overlook the decoupled structure.
Consequently, \method’s advantage is largest at 25\% retention, where dimension-aware scoring becomes most critical. Figure~\ref{fig:exp1} further shows that PyramidDrop$^\dagger$ and SparseVLM$^\dagger$ preserve broader spatial structure but leave noticeable gaps, whereas \method focuses on semantic regions while maintaining spatial coverage.

\subsection{Ablation Studies}
\label{sec:ablation}

\paragraph{Uniform vs.\ Pareto-adaptive compression.}
\begin{wraptable}{r}{0.63\linewidth}
\centering
\vspace{-13pt}
\caption{\textbf{Uniform vs.\ Pareto-adaptive compression on Ours-2B.} ``Pareto'' uses the PACS-optimized schedule; ``Uniform'' applies equal retention at every stage.}
\label{tab:uniform_rate}
\begin{adjustbox}{max width=\linewidth}
\begin{tabular}{ll cccccc}
\toprule
\textbf{Schedule} & \textbf{Ret.} & \textbf{MMB} & \textbf{AI2D} & \textbf{MME} & \textbf{HalBench} & \textbf{SEED} & \textbf{RWQA} \\
\midrule
\rowcolor{gray!10} Full tokens & 100\% & 69.7 & 80.8 & 2084.4 & 42.72 & 71.8 & 61.8 \\
\midrule
Pareto & 25\% & \positive{+1.4\%} \real{(70.7)} & \positive{+2.9\%} \real{(83.1)} & \positive{+0.4\%} \real{(2092)} & \negative{-2.5\%} \real{(41.6)} & \negative{-0.6\%} \real{(71.4)} & \negative{-1.0\%} \real{(61.2)} \\
Uniform & 25\% & \negative{-2.2\%} \real{(68.2)} & \negative{-4.7\%} \real{(77.0)} & \negative{-5.2\%} \real{(1976)} & \negative{-14.8\%} \real{(36.4)}  & \negative{-6.4\%} \real{(67.2)} & \negative{-8.4\%} \real{(56.6)} \\
\midrule
Pareto & 50\% & \positive{+2.2\%} \real{(71.2)} & \positive{+4.0\%} \real{(84.0)} & \positive{+1.1\%} \real{(2107)} & \positive{+0.5\%} \real{(42.9)} & \positive{+1.0\%} \real{(72.5)} & \positive{+1.3\%} \real{(62.6)} \\
Uniform & 50\% & \negative{-0.3\%} \real{(69.5)} & \positive{+1.7\%} \real{(82.2)} & \negative{-1.6\%} \real{(2052)} & \negative{-6.5\%} \real{(40.0)} & \negative{-0.6\%} \real{(71.4)} & \negative{-1.8\%} \real{(60.7)} \\
\midrule
Pareto & 75\% & \positive{+3.2\%} \real{(71.9)} & \positive{+5.1\%} \real{(84.9)} & \positive{+1.7\%} \real{(2120)} & \positive{+3.9\%} \real{(44.4)} & \positive{+1.8\%} \real{(73.1)} & \positive{+1.0\%} \real{(62.4)} \\
Uniform & 75\% & \positive{+1.9\%} \real{(71.0)} & \positive{+4.3\%} \real{(84.3)} & \positive{+0.1\%} \real{(2087)} & \negative{-4.5\%} \real{(40.8)} & \positive{+1.1\%} \real{(72.6)} & \positive{+0.5\%} \real{(62.1)}  \\
\bottomrule
\end{tabular}
\end{adjustbox}
\vspace{-13pt}
\end{wraptable}

Table~\ref{tab:uniform_rate} compares uniform compression (equal retention across stages) against the PACS-optimized schedule. For the same overall retention target, Pareto scheduling consistently outperforms uniform allocation. 
At 25\% overall retention, the final Pareto configuration $(0.75, 0.67, 0.50)$ outperforms uniform by 3.6\%  on MMB and 7.6\% on AI2D. 
This validates that adaptive budget allocation guided by layer-wise discriminability is critical.

\paragraph{SDIE internal ablation.}
\begin{wraptable}{r}{0.5\linewidth}
\centering
\vspace{-12pt}
\caption{\textbf{SDIE internal ablation.} Setting~1 = full \method; Setting~2 = $s_i^T \times s_i^{H,W}$; Setting~3 = $s_i^T \times s_i^H$ only; Setting~4 = $s_i^T \times s_i^W$ only.}
\label{tab:ablation_sdie}
\begin{adjustbox}{max width=\linewidth}
\begin{tabular}{llccc}
\toprule
\textbf{Configuration} & \textbf{Ret.} & \textbf{MMB} & \textbf{AI2D} & \textbf{MME} \\
\midrule
\rowcolor{gray!10} Full tokens & 100\% & 69.7 & 80.8 & 2084 \\
\midrule
\multirow{3}{*}{Full \method (Setting 1)} & 25\% & \positive{+1.4\%} \real{(70.7)}  & \positive{+2.9\%} \real{(83.1)}  & \positive{+0.4\%} \real{(2092)}  \\
 & 50\% & \positive{+2.2\%} \real{(71.2)} & \positive{+4.0\%} \real{(84.0)} & \positive{+0.6\%} \real{(2098)} \\
 & 75\% & \positive{+3.2\%} \real{(71.9)} & \positive{+5.1\%} \real{(84.9)} & \positive{+1.7\%} \real{(2120)} \\
\midrule
\multirow{3}{*}{$s_i^T \times s_i^{H,W}$ (Setting 2)} & 25\% & \negative{-2.4\%} \real{(68.0)} & \negative{-2.5\%} \real{(78.8)} & \negative{-3.7\%} \real{(2008)} \\
 & 50\% & \negative{-0.1\%} \real{(69.6)} & \positive{+3.0\%} \real{(83.2)} & \negative{-0.2\%} \real{(2080)} \\
 & 75\% & \positive{+0.9\%} \real{(70.3)} & \positive{+5.1\%} \real{(84.9)} & \positive{+0.5\%} \real{(2096)} \\
\midrule
\multirow{3}{*}{$s_i^T \times s_i^H$ (Setting 3)} & 25\% & \negative{-2.4\%} \real{(68.0)} & \negative{-3.5\%} \real{(78.0)} & \negative{-2.6\%} \real{(2031} \\
 & 50\% & \positive{+1.2\%} \real{(70.5)} &  \positive{+1.4\%} \real{(81.9)} &  \negative{-0.7\%} \real{(2069)} \\
 & 75\% & \positive{+0.7\%} \real{(70.2)} & \positive{+4.8\%} \real{(84.7)} & \positive{+0.5\%} \real{(2095)} \\
\midrule
\multirow{3}{*}{$s_i^T \times s_i^W$ (Setting 4)} & 25\% & \negative{-4.6\%} \real{(66.5)} & \negative{-2.0\%} \real{(79.2)} & \negative{-2.7\%} \real{(2029)} \\
 & 50\% & \positive{+0.6\%} \real{(70.1)} & \positive{+2.5\%} \real{(82.8)} & \negative{-1.2\%} \real{(2059)} \\
 & 75\% & \negative{+0.0\%} \real{(69.7)} & \positive{+5.5\%} \real{(85.2)} & \negative{-0.4\%} \real{(2075)} \\
\bottomrule
\vspace{-42pt}
\end{tabular}
\end{adjustbox}
\end{wraptable}

Table~\ref{tab:ablation_sdie} compares alternative importance formulations on the our backbone. The full configuration (Setting~1: $(s_i^H + s_i^W) \times s_i^T$ with all SDIE components) consistently outperforms alternatives. Setting~2 (only $s_i^T \times s_i^{H,W}$, without spatial axis decoupling. $s_i^{HW}$ is a single un-decoupled spatial score, \emph{not} $s_i^H \times s_i^W$) notably drops MMB by 2\% - 4\%, while Setting~4 (only $s_i^T \times s_i^W$, discarding $s_i^H$) causes a much sharper decline at 25\% retention, demonstrating that both spatial dimensions are jointly essential.

\begin{figure}[t]
\centering
\includegraphics[width=\linewidth]{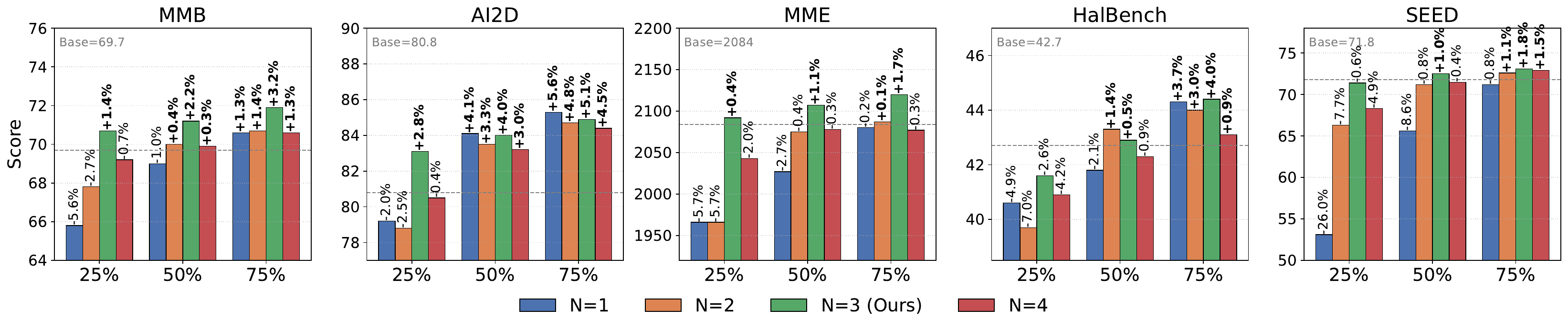}
\caption{\textbf{Effect of the number of compression stages ($N$) on Ours-2B.} The three-stage design ($N\!=\!3$) consistently preserves or improves baseline performance across all retention ratios, while $N\!=\!1$ incurs notable degradation under aggressive compression.}
\label{fig:num_stages}
\end{figure}

\paragraph{Number of compression stages.}
Figure~\ref{fig:num_stages} compares $N \in {1,2,3,4}$ stages at the same token retention. A single stage ($N=1$) causes abrupt information loss (e.g., SEED drops 26.0\% at 25\% retention), while $N=2$ remains insufficient under aggressive compression. A fourth stage ($N=4$) brings negligible gains with additional scheduling complexity. Overall, setting $N=3$ provides the best trade-off, matching or outperforming the baseline across benchmarks and retention levels, while avoiding unnecessary complexity from additional compression stages.

\subsection{Efficiency Analysis}
\label{sec:efficiency}

We evaluate inference efficiency on a single H20-96G GPU at $1024^2$ resolution (=1024 visual tokens after $2\times2$ patch merge). Table~\ref{tab:efficiency} reports batch-size1 results; matched-retention batched serving is in Appendix\ref{app:batched_efficiency}.
At 25\% retention ($\sim$256 tokens), \method reduces FLOPs by $2.0\times$ and KV cache from 252,MB to 72,MB, while improving the average score from 65.4 to 65.6.
Since batch-1 decode is dominated by weight loading, all methods have the same $1.00\times$ decode speed, and KV-cache reduction is generic to 256-token methods. \method’s key advantage is accuracy at matched serving cost: 65.6 average, outperforming SparseVLM by $1.0$ and PyramidDrop by $2.9$, while exceeding the full model (65.4). Under batched decode at 25\% retention, \method achieves up to $1.9\times$ faster decoding and $2.0\times$ lower peak memory, and supports batch64 where the full model OOMs (Appendix\ref{app:batched_efficiency}).
\begin{table}[t]
\centering
\caption{\textbf{Efficiency analysis of different compression
  methods on Ours 2B model.}
  Measured on a single H20-96G GPU, batch size 1,
  input resolution $1024^2$.
  \colorbox{pink!15}{\strut Attention-based} methods shown
  with red background;
  \colorbox{green!10}{\strut spatially-decoupled} methods
  with green background.
  Speedup ratios relative to the full-token baseline in
  parentheses. Avg.\ Score is averaged over MMB, AI2D, HallBench, RealWorld and SEEDBench-IMG five benchmarks.}
\label{tab:efficiency}
\begin{adjustbox}{max width=\linewidth}
\begin{tabular}{ll cccccr}
\toprule
\textbf{Method} & \textbf{\# Token} &
\textbf{FLOPs (T)} &
\textbf{Prefill (ms)} &
\textbf{Decode (ms/tok)} &
\textbf{KV Cache (MB)} &
\textbf{GPU Mem (GB)} &
\textbf{Avg.\ Score} \\
\midrule
Ours-2B (Full) & 1024
  & 5.43 & 127.8 & 77.4
  & 251.7 & 7.1 & 65.4 \\
\midrule
\multicolumn{8}{c}{\textit{Retention $=$ 75\%
  (\texttildelow768 tokens)}} \\
\midrule
\rowcolor{pink!10}
\rowcolor{pink!10}
SparseVLM & 768
  & 4.36 ($\times$1.25) & 103.1 ($\times$1.24)
  & 77.8 ($\times$1.00)
  & 191.7 & 6.7 & 66.6 \\
\rowcolor{pink!10}
PyramidDrop$^\dagger$ & 768
  & 4.36 ($\times$1.25) & 106.1 ($\times$1.20)
  & 78.7 ($\times$0.98)
  & 191.7 & 6.7 & 66.4 \\
\rowcolor{green!8}
\textbf{\method} & \textbf{768}
  & \textbf{4.60} ($\times$\textbf{1.18})
  & \textbf{105.6} ($\times$\textbf{1.20})
  & \textbf{77.6} ($\times$\textbf{1.00})
  & 191.7  & \textbf{6.8}
  & \textbf{67.3} \\
\midrule
\multicolumn{8}{c}{\textit{Retention $=$ 50\%
  (\texttildelow512 tokens)}} \\
\midrule
\rowcolor{pink!10}
\rowcolor{pink!10}
SparseVLM & 512
  & 3.56 ($\times$1.53) & 98.8 ($\times$1.29)
  & 75.9 ($\times$1.02)
  & 131.7 & 6.3 & 66.2 \\
\rowcolor{pink!10}
PyramidDrop$^\dagger$ & 512
  & 3.56 ($\times$1.53) & 96.9 ($\times$1.32)
  & 73.4 ($\times$1.05)
  & 131.7 & 6.3 & 66.2 \\
\rowcolor{green!8}
\textbf{\method} & \textbf{512}
  & \textbf{3.61} ($\times$\textbf{1.50})
  & \textbf{102.1} ($\times$\textbf{1.25})
  & \textbf{73.6} ($\times$\textbf{1.05})
  & 131.7 & \textbf{6.4}
  & \textbf{66.6} \\
\midrule
\multicolumn{8}{c}{\textit{Retention $=$ 25\%
  (\texttildelow256 tokens)}} \\
\midrule
\rowcolor{pink!10}
\rowcolor{pink!10}
SparseVLM & 256
  & 2.74 ($\times$1.98) & 98.2 ($\times$1.30)
  & 76.9 ($\times$1.01)
  & 71.7 & 6.0 & 64.6 \\
\rowcolor{pink!10}
PyramidDrop$^\dagger$ & 256
  & 2.74 ($\times$1.98) & 98.2 ($\times$1.30)
  & 76.6 ($\times$1.01)
  & 71.7 & 6.0 & 62.7 \\
\rowcolor{green!8}
\textbf{\method} & \textbf{256}
  & \textbf{2.74} ($\times$\textbf{1.98})
  & \textbf{99.6} ($\times$\textbf{1.28})
  & \textbf{77.3} ($\times$\textbf{1.00})
  & 71.7 & \textbf{6.2}
  & \textbf{65.6} \\
\bottomrule
\end{tabular}
\end{adjustbox}
\end{table}

\section{Conclusion}

We have presented \method, the first visual token compression framework designed for native MLLM built on multi-axis positional encoding. 
Through systematic probing analysis, we reveal the progressive redundancy accumulation pattern induced by the IFB mechanism and the cross-dimensional heterogeneous distribution of redundancy. Building on these findings, \method leverages the inherent dimensional decoupling of Native-RoPE to construct a direction-aware importance measure and employs Pareto front traversal to determine the optimal compression ratio at each layer. Experiments across multiple benchmarks validate the effectiveness and efficiency of \method, demonstrating that the structured priors inherent in native architectures yield systematic gains for efficient inference.


\bibliographystyle{splncs04}
\bibliography{reference}

\clearpage
\newpage

\appendix
\section{Supplementary Material}

In this supplementary material, we provide:
(A)~related work (\S\ref{sec:related});
(B)~theoretical analysis and proofs for SDIE and PACS (\S\ref{app:theory});
(C)~per-stage training results comparing original and compressed models (\S\ref{app:per_stage});
(D)~detailed Pareto front construction and compression configurations (\S\ref{app:pareto_config});
(E)~additional qualitative case studies comparing \method with transferred baselines (\S\ref{app:cases});
(F)~extended experimental results and ablations (\S\ref{app:extended_exp});
(G)~model architecture details, training data composition, and hyperparameter configurations (\S\ref{sec:arch_details},~\S\ref{sec:training_appendix}).

\subsection{Related Work}
\label{sec:related}

\paragraph{Multimodal Large Language Models.}

Multimodal learning spans diverse application domains and also addresses cross modal semantic depth mismatch through AHRA's hierarchical expert selection~\citep{meng2026adaptive} and short video content governance through reasoning enhanced domain adaptive pretraining~\citep{wang2025reasoning}. SP-Mind automates spatial proteomics analysis through an autonomous reasoning agent~\citep{yuan2026sp}. PDU-Net combines path aggregation and dual attention for lung nodule segmentation~\citep{chang2025research}.
Modular MLLMs connect an independent ViT with a language model. The LLaVA series~\citep{liu2023visual,liu2024improved,liu2024llavanext}, InternVL~\citep{chen2024internvl,chen2024expanding,zhu2025internvl3}, Qwen-VL~\citep{team2023qwen,wang2024qwen2}, and LLaVA-OneVision~\citep{li2024llava,li2024llava} show strong capabilities across visual question answering, video understanding, and related tasks, establishing the ``ViT + projection + LLM'' pipeline as the dominant paradigm~\citep{dong2024internlm,dai2024nvlm}. This design, however, couples an externally pretrained encoder with the language model, incurring representational misalignment at the modality boundary and extra deployment complexity. Native MLLMs instead discard the ViT and jointly model visual and linguistic sequences within a unified autoregressive Transformer, offering a simpler and more scalable formulation. Pioneered by Fuyu~\citep{bavishi2023fuyu} and EVE~\citep{diao2024unveiling}, this direction is advanced by Chameleon~\citep{team2024chameleon}, Emu3~\citep{wang2024emu3}, Show-o~\citep{xie2024show}, VILA-U~\citep{wu2024vila}, Emu3.5~\citep{cui2025emu3}, SAIL~\citep{sail25}, and NEO~\citep{diao2025pixels}, collectively validating the paradigm~\citep{zhao2025unified}. Among them, NEO introduces the Pre-Buffer and Post-LLM layered design~\citep{diao2025pixels}, UniMoD~\citep{mao2025unimod} explores a mixed-depth perspective for unified understanding and generation, and The Narrow Gate~\citep{serra2024narrow} reveals the localized nature of cross-modal information flow. Yet since every pixel patch is tokenized directly without a ViT's spatial downsampling, native MLLMs accumulate far longer visual sequences, and a dedicated token compression framework for them remains absent.

\paragraph{Visual Token Compression.}

Visual token compression reduces redundancy in the visual sequences of MLLMs and has yielded a rich body of methods under modular architectures. FastV~\citep{chen2024image} first prunes low-contribution tokens via text-to-vision attention. Subsequent work extends this along several axes: PyramidDrop~\citep{xing2024pyramiddrop}, HiDrop~\citep{wu2026hidrop}, and AdaptInfer~\citep{zhang2025adaptinfer} introduce progressive or dynamically scheduled multi-stage compression; SparseVLM~\citep{zhang2024sparsevlm} and FitPrune~\citep{ye2025fit} enable training-free acceleration; LLaVA-PruMerge~\citep{shang2025llava} jointly prunes and merges; DivPrune~\citep{alvar2025divprune}, CDPruner~\citep{zhang2025beyond}, HoloV~\citep{zou2025don}, and FSR~\citep{tong2026focus} select tokens from a semantic-coverage view; VisionZip~\citep{yang2025visionzip} and MMTok~\citep{dong2025mmtok} cast coverage maximization as an objective; VisionTrim~\citep{yu2026visiontrim} unifies encoding- and decoding-stage compression; and IDPruner~\citep{tan2026idpruner}, AgilePruner~\citep{baek2026agilepruner}, D$^{2}$Pruner~\citep{zhang2026d2pruner}, and ViTCoP~\citep{luo2026vitcop} jointly optimize importance and diversity. Orthogonally, ADSC~\citep{deniz2025vision} applies uniform layer-wise downsampling, while EvoComp~\citep{song2026evocomp} trains a learned compressor via offline evolutionary search. Across these methods, importance signals are positional-agnostic or single-axis, and scheduling relies on layer-level heuristics rather than an explicit multi-objective trade-off. Although transplantable to native backbones (our baselines in Table~\ref{tab:compression_methods}), they neither exploit the decoupled multi-axis positional structure of native MLLMs nor jointly optimize performance, efficiency, and coverage. The proposed \method is the first to frame native visual token compression as Pareto multi-objective optimization, exploiting multi-axis structure via SDIE and searching schedules via PACS.

\subsection{Training Data and Hyperparameter Details}
\label{sec:training_appendix}

\paragraph{Pre-training.} The model learns basic visual concepts from large-scale image-text paired data. The training corpus comprises approximately 518M image-text pairs: BLIP3o~\citep{chen2025blip3} (29M), COYO~\citep{byeon2022coyo} (314M, 10\% sampled), Recap-DataComp~\citep{li2024recaption} (610M, 50\% sampled), and Relaion-COCO~\citep{schuhmann2022laioncoco, schuhmann2022laion5b} (333M, 50\% sampled). During this stage, only the Patch Embedding layers, Pre-Buffer, and the newly added H/W-directional QK linear layers and normalization parameters in the Post-LLM are optimized; the Post-LLM backbone parameters remain frozen.

\paragraph{Mid-training.} This stage strengthens vision-language alignment and enhances recognition of high-resolution images, complex scenes, and dense OCR content, bridging the gap between coarse pre-training concepts and downstream instruction data. The training data comprises approximately 89M samples drawn from 8 heterogeneous sources, employing a concept-balanced sampling strategy based on MetaCLIP~\citep{xu2023demystifying} with re-captioning by large VLMs to improve caption fidelity and reduce concept imbalance. In this stage all model parameters, including the Post-LLM backbone, are unfrozen and optimized end-to-end.

\paragraph{Supervised fine-tuning.} High-quality multi-source instruction data is used to enhance instruction following and multi-turn conversational abilities. The data spans a broad range of task categories, including Captioning, General QA, Science, Chart, Mathematics, Knowledge, OCR, Document, Grounding, and Conversation, so as to cover both perception-heavy and reasoning-heavy scenarios (full source and proportion details in Table~\ref{tab:datasets_sft}).

\paragraph{Reinforcement learning.} A total of 12{,}000 samples are used for reinforcement learning, covering Document/OCR (16\%), Chart Understanding (10\%), General VQA (10\%), Conversation/Instruction (12\%), Mathematics/Geometry (15\%), Captioning (10\%), Grounding (5\%), Science (20\%), and Visual Reasoning (14\%). To prioritize informative samples, we apply pass@6 filtering: each candidate is sampled 6 times, and only those with 1--5 correct responses (out of 6) are retained, while samples with 6/6 correct (trivial) or 0/6 correct (overly difficult) are discarded to maximize the learning signal (details in Table~\ref{tab:datasets_rl}).

\begin{table}[htbp]
\centering
\caption{\textbf{Pre-training data summary.} Total: $\sim$518M image-text pairs.}
\label{tab:datasets_pt}
\resizebox{\linewidth}{!}{
\begin{tabular}{lcccc}
\toprule
\textbf{Dataset} & \textbf{Language} & \textbf{Size} & \textbf{Sampling Strategy} & \textbf{Notes} \\
\midrule
BLIP3o~\citep{chen2025blip3} & EN & 29.4M & full sampling & High-quality synthetic captions \\
COYO~\citep{byeon2022coyo} & EN & 314M & partly sampling & Web-crawled, noisy \\
Recap-DataComp~\citep{li2024recaption} & EN & 610M & partly sampling & High-quality recaptions \\
Relaion-COCO~\citep{schuhmann2022laioncoco, schuhmann2022laion5b} & EN & 333M & partly sampling & High quality recaptions \\
ImageNet-21K~\citep{deng2009imagenet} & EN & 10.2M & partly sampling & Classification labels as short captions \\
\midrule
\textbf{Total} & & \textbf{$\sim$518M} & \\
\bottomrule
\end{tabular}
}
\end{table}

\begin{table}[htbp]
\centering
\caption{\textbf{Mid-training data summary.} Total: $\sim$89M samples.}
\label{tab:datasets_mt}
\resizebox{\linewidth}{!}{
\begin{tabular}{lcccc}
\toprule
\textbf{Dataset} & \textbf{Language} & \textbf{Size} & \textbf{Sampling Strategy} & \textbf{Notes} \\
\midrule
ImageNet-21K~\citep{deng2009imagenet} & EN, CH & 1.6M & partly sampling & Classification labels as short captions \\
laionCN~\citep{an2025llava} & EN, CH & 6.4M & partly sampling & English \& Chinese web-crawled image–text pairs \\
COYO~\citep{byeon2022coyo} & EN, CH & 15.1M & partly sampling & Web-crawled, noisy \\
Recap-DataComp~\citep{li2024recaption} & EN, CH & 6.2M & partly sampling & High-quality recaptions, rewritten for dense semantics \\
SAM-1B~\citep{kirillov2023segment} & EN, ZH & 1.1M & partly sampling & Segmentation-grounded dense captions \\
MINT~\citep{awadalla2024mint} & EN, CH & 4.1M & partly sampling & Multimodal interleaved documents \\
Zero250M~\citep{xie2022zero} & EN, CH & 9.6M & partly sampling & Large-scale image–text corpus \\
Obelics~\citep{laurenccon2023obelics} & EN, CH & 45.4M & partly sampling & Interleaved image–text web documents \\
\midrule
\textbf{Total} & & & \textbf{$\sim$89M} & \\
\bottomrule
\end{tabular}
}
\end{table}

\begin{table}[htbp]
\centering
\caption{\textbf{SFT data summary.} Multi-source instruction data spanning 8 task categories. Total:~12M samples.}
\label{tab:datasets_sft}
\resizebox{\linewidth}{!}{
\begin{tabular}{c|c}
\toprule
\textbf{Task Category} & \textbf{Datasets} \\
\midrule
Captioning & TextCaps (en)~\citep{sidorov2020textcaps}, ShareGPT4V (en\&zh)~\citep{chen2024sharegpt4v}, COCO~\citep{lin2014microsoft} \\
\rowcolor{gray!15}
General QA & VQAv2~\citep{goyal2017making}, GQA~\citep{hudson2019gqa}, OKVQA~\citep{marino2019ok}, A-OKVQA~\citep{schwenk2022okvqa}, VSR~\citep{shaharabany2022looking}, VisualDialog~\citep{das2017visual}, ViQuAE~\citep{lerner2022viquae}  \\
Science & AI2D~\citep{kembhavi2016diagram}, ScienceQA~\citep{lu2022learn}, TQA~\citep{kembhavi2017you} \\
\rowcolor{gray!15}
Chart & ChartQA~\citep{masry2022chartqa}, DVQA~\citep{kafle2018dvqa}, PlotQA~\citep{methani2020plotqa}, IconQA~\citep{lu2021iconqa}, VisText~\citep{tang2023vistext}, Chart-to-Text~\citep{kantharaj2022chart}, FigureQA~\citep{kahou2017figureqa} \\
Mathematics & Geo170K~\citep{gao2023g}, GeoQA+~\citep{chen2021geoqa}, TabMWP~\citep{lu2022dynamic}, CLEVR-Math~\citep{lindstrom2022clevr}, Geometry3K~\citep{lu2021inter} \\
\rowcolor{gray!15}
OCR & OCRVQA~\citep{mishra2019ocr}, TextVQA~\citep{singh2019towards}, InfoVQA~\citep{mathew2022infographicvqa}, DocVQA~\citep{mathew2021docvqa}, ST-VQA~\citep{biten2019scene}, SynthText~\citep{gupta2016synthetic} \\
Grounding & RefCOCO/+/g~\citep{yu2016modeling}, Visual Genome~\citep{krishna2017visual}, Visual7W~\citep{zhu2016visual7w} \\
\rowcolor{gray!15}
Conversation & LLaVA-150K~\citep{liu2023visual}, LVIS-Instruct4V~\citep{wang2023see}, ALLaVA~\citep{chen2024allava} \\
\bottomrule
\end{tabular}
}
\end{table}

\begin{table}[htbp]
\centering
\caption{\textbf{RL data category proportions.} Total: $\sim$12K samples.}
\label{tab:datasets_rl}
\begin{adjustbox}{max width=\linewidth}
\begin{tabular}{lrl}
\toprule
\textbf{Category} & \textbf{Sampling\%} & \textbf{Representative Datasets} \\
\midrule
Document/OCR & 14\% & DocVQA~\citep{mathew2021docvqa}, TextVQA~\citep{singh2019towards}, ESTVQA~\citep{wang2020general} \\
Chart Understanding & 9\% & ChartQA~\citep{masry2022chartqa}, ChartX~\citep{xia2402chartx}, TableVQA~\citep{kim2024tablevqa} \\
General VQA & 9\% & VQAv2~\citep{goyal2017making}, AOKVQA~\citep{schwenk2022okvqa}, GQA~\citep{hudson2019gqa}, POPE~\citep{li2023evaluating} \\
Conversation/Instruction & 11\% & Innovator-VL~\citep{wen2026innovator}, Thyme~\citep{zhang2025thyme} \\
Mathematics/Geometry & 13\% & GeoQA~\citep{chen2021geoqa}, GSM8K~\citep{cobbe2021training}, TabMWP~\citep{lu2022dynamic}, MMMath~\citep{sun2024mm} \\
Captioning & 9\% & CLEVR-Counting~\citep{johnson2017clevr}, TallyQA~\citep{acharya2019tallyqa} \\
Grounding & 4\% & Orsta~\citep{ma2025one} \\
Science & 18\% & ScienceQA~\citep{lu2022learn}, SciVQA~\citep{borisova2025scivqa} \\
Visual Reasoning & 13\% & MMMU~\citep{yue2024mmmu}, AI2D~\citep{kembhavi2016diagram} \\
\bottomrule
\end{tabular}
\end{adjustbox}
\end{table}

\subsection{Training Hyperparameters}
\label{sec:hyperparameters}
Table~\ref{tab:hyperparam} provides the complete hyperparameter configurations across all four training stages. All stages share the AdamW optimizer with identical momentum and epsilon settings, a cosine learning rate schedule, weight decay of 0.01 and bfloat16 numerical precision.
The maximum sequence length is 16{,}384, 32{,}768, 32{,}768 and 8{,}192 tokens in four stages, respectively.
The key differences across stages lie in the peak learning rate, which decreases from $8{\times}10^{-4}$ in pre-training to $1{\times}10^{-5}$ in RL as the model converges, and the global batch size, gradually reduced from 2{,}560 to 128 to match decreasing data volume and enable finer gradient updates. Input resolution is $256^2$–$1024^2$ during pre-training and extends to $256^2$–$2048^2$ in subsequent stages. A text-only data ratio of 0.3 is applied during pre-training and mid-training to preserve the backbone’s general language capability.

\begin{table}[htbp]
\centering
\caption{\textbf{Hyperparameter configurations for each training stage.}}
\label{tab:hyperparam}
\begin{tabular}{l cccc}
\toprule
\textbf{Configuration} & \textbf{Pre-training} & \textbf{Mid-training} & \textbf{SFT} & \textbf{RL} \\
\midrule
Resolution & $112^2$\textasciitilde$1024^2$ & $256^2$\textasciitilde$1024^2$ & $256^2$\textasciitilde$2048^2$ & $256^2$\textasciitilde$2048^2$ \\
Optimizer & \multicolumn{4}{c}{AdamW ($\beta_1=0.9, \beta_2=0.999, \epsilon=10^{-8}$)} \\
Weight decay & \multicolumn{4}{c}{0.01} \\
LR schedule & cosine & cosine & cosine & cosine \\
Peak LR & $5 \times 10^{-4}$ & $5 \times 10^{-5}$ & $5 \times 10^{-5}$ & $1 \times 10^{-5}$ \\
Training steps & 250K & 46K & 19K & 2K \\
Warmup steps & 1K & 200 & 200 & 100 \\
Max sequence length & 16,384 & 32,768 & 32,768 & 8,192 \\
Global batch size & 2,560 & 1,920 & 640 & 128 \\
Numerical precision & \multicolumn{4}{c}{bfloat16} \\
\bottomrule
\end{tabular}
\end{table}

\subsection{Model Architecture Details}
\label{sec:arch_details}

Our native backbone, which we refer to as ``Ours'' throughout the paper, is a ViT-free multimodal large language model built on Qwen3-1.7B~\citep{yang2025qwen3}.
Unlike modular designs that route image features through a separate visual encoder before feeding them into the LLM, our model directly tokenizes raw image patches via a lightweight convolutional Patch Embedding layer and processes them jointly with text tokens in a single autoregressive Transformer, eliminating the modality boundary entirely.

\paragraph{Overall structure.}
The model comprises three functional blocks.
The \textbf{Patch \& Word Embedding} front-end converts each $32{\times}32$ image patch into a visual token through a convolutional Patch Embedding layer, which is then downsampled by a Dense Embedding module to match the LLM hidden dimension $d{=}2048$; text tokens are embedded via the standard word embedding table shared with the LLM.
The \textbf{Pre-Buffer} (Initial Fusion Buffer, IFB) consists of the first $L_{\text{buf}}{=}4$ Transformer layers, where cross-modal semantic alignment and Native-RoPE position registration take place; within the IFB, visual token discriminability remains near zero (see probing analysis in \S\ref{sec:overview}), and no compression is applied.
The \textbf{Post-LLM body} comprises the remaining $L_{\text{post}}{=}36$ layers, which inherit the pre-trained Qwen3-1.7B weights and are augmented with additional H/W-directional QK projections for spatially decoupled attention; \method compression stages are placed within this body.
The total depth is $L_{\text{buf}} + L_{\text{post}} = 40$ layers, with ${\sim}2.2$B parameters.

\paragraph{Attention mechanism.}
We adopt a \emph{mixed} masking strategy: visual tokens attend \emph{bidirectionally} within the same image, enabling full spatial context exchange, while text tokens follow the standard \emph{autoregressive} causal mask. Cross-modal attention is \emph{causal}---text tokens can attend to preceding visual tokens, but not vice versa. This is implemented via FlexAttention with a custom block mask that combines document-level isolation, per-image bidirectional attention, and text-side causal masking.

\paragraph{Native-RoPE.}
Each token is assigned a three-component position ID $(t, h, w)$ corresponding to its sequence position, spatial height, and spatial width, respectively, thereby explicitly preserving both sequential and two-dimensional spatial information. Distinct rotary base frequencies ($\beta_T{=}10^6$, $\beta_H{=}\beta_W{=}10^4$) are applied independently to the three components. The QK projection is further extended with an additional H/W-directional linear layer: the original QK head dimension is 64 for the $T$ component, while an additional 64-dimensional head is introduced and evenly divided into 32 dimensions for H and 32 for W, with each component equipped with its own RMS normalization. This structural design explicitly separates temporal and spatial positional interactions, providing the dimensional decoupling that SDIE (\S\ref{sec:sdie}) directly exploits for direction-aware importance estimation.

\paragraph{Architecture configuration.}
Table~\ref{tab:arch_details} summarizes the full configuration.

\begin{table}[htbp]
\centering
\caption{\textbf{Architecture configuration of our trained model.}}
\label{tab:arch_details}
\begin{tabular}{lc}
\toprule
\textbf{Configuration} & \textbf{Ours-2B} \\
\midrule
Backbone LLM & Qwen3-1.7B \\
Pre-Buffer layers ($L_{\text{buf}}$) & 4 \\
Post-LLM layers ($L_{\text{post}}$) & 36 \\
Total layers & 40 \\
Hidden dim.\ $d$ & 2048 \\
Attention heads & 16 (QK) / 8 (V) \\
Original QK head dim.\ (T) & 64 \\
Added QK head dim.\ (H+W) & 32+32 \\
FFN intermediate dim. & 6144 \\
Patch size & $16 \times 16$ \\
$\beta_P$ / $\beta_H$ / $\beta_W$ & $10^6$ / $10^4$ / $10^4$ \\
Attention type & Mixed (visual bidirectional + text autoregressive) \\
\midrule
Total parameters & $\sim$2.2B \\
\bottomrule
\end{tabular}
\end{table}

\subsection{Evaluation Details}
\label{sec:eval_details}

All benchmark evaluations are conducted using VLMEvalKit~\citep{duan2024vlmevalkit} with default settings (prompts, answer extraction, and scoring). 
We employ the \texttt{gpt-4o-1120} API as LLM judge.
The greedy decoding (temperature $= 0$, \texttt{topK=1}) is used to ensure reproducibility. For MME, we report the sum of perception and reasoning scores. All experiments are run on a single H20-96G GPU.

\section{Theoretical Analysis}
\label{app:theory}

\subsection{Optimality of SDIE Additive-Multiplicative Fusion}
\label{app:theory_Optimality}

We provide a formal justification for the spatial-additive, positional-multiplicative fusion strategy in Eq.~\ref{eq:sdie_fusion}.

\begin{theorem}[Optimality of Additive-Multiplicative Fusion]
\label{thm:sdie_optimal}
Let $s_i^H, s_i^W \geq 0$ be the height and width saliency scores for visual token $i$, and $s_i^T \geq 0$ be its temporal (positional) relevance. Define the token importance as $s_i = f(s_i^H, s_i^W, s_i^T)$. Restrict to the class of separable functions $f = g(s_i^H, s_i^W) \cdot h(s_i^T)$ where $g$ is bilinear in its arguments and $h$ is linear. Under the following desiderata:
\textit{i)} \textbf{Uni-directional preservation}: if $s_i^H \gg 0$ but $s_i^W \approx 0$ (or vice versa), the token should be retained;
\textit{ii)} \textbf{Global irrelevance suppression}: if $s_i^T \approx 0$, the token should be discarded regardless of spatial saliency;
\textit{iii)} \textbf{Monotonicity}: $f$ is non-decreasing in each argument; then, under the additional parsimony restriction that discards the redundant cross term ($\gamma=0$) and the Native-RoPE spatial-symmetry prior ($\alpha=\beta$), $f(s_i^H, s_i^W, s_i^T) = (s_i^H + s_i^W) \times s_i^T$ is the unique (up to positive scaling) form satisfying all three conditions within this function class.
\end{theorem}

\begin{proof}
We proceed by elimination of alternative functional forms within the separable bilinear class.

\textbf{Step 1: Spatial components must combine additively.}
Consider the general class of separable functions $f = g(s_i^H, s_i^W) \cdot h(s_i^T)$, where $g$ is bilinear, \ie, $g(s_i^H, s_i^W) = \alpha s_i^H + \beta s_i^W + \gamma s_i^H s_i^W$ with $\alpha, \beta, \gamma \geq 0$. Any purely multiplicative form ($\alpha = \beta = 0, \gamma > 0$) yields $g(a, 0) = g(0, b) = 0$ for all $a, b > 0$, violating desideratum (i), so we must have $\alpha, \beta > 0$. The cross term $\gamma s_i^H s_i^W$ vanishes whenever either component is zero, so it contributes nothing to uni-directional preservation (i); it can only rescale tokens that are already salient on both axes, which is not required by any of (i)--(iii). We therefore invoke the parsimony restriction stated in the theorem and set $\gamma = 0$, keeping the minimal form that meets the desiderata. This gives $g(s_i^H, s_i^W) = \alpha s_i^H + \beta s_i^W$ with $\alpha, \beta > 0$. We emphasize that uniqueness is asserted only under this parsimony restriction together with the symmetry prior below; without them, any $\gamma \ge 0$ and any $\alpha,\beta>0$ also satisfy (i)--(iii).

\textbf{Step 2: Temporal component must act as a multiplicative gate.}
Condition (ii) requires $f = 0$ whenever $s_i^T = 0$, regardless of $s_i^H, s_i^W$. This demands $h(0) = 0$. Together with monotonicity (iii) and non-negativity, and restricting $h$ to be linear, we have $h(s_i^T) = s_i^T$ (up to positive scaling).

\textbf{Step 3: Verification of monotonicity.}
For the resulting $f = (\alpha s_i^H + \beta s_i^W) \cdot s_i^T$:
\begin{equation}
\frac{\partial f}{\partial s_i^H} = \alpha \cdot s_i^T \geq 0, \quad
\frac{\partial f}{\partial s_i^W} = \beta \cdot s_i^T \geq 0, \quad
\frac{\partial f}{\partial s_i^T} = \alpha s_i^H + \beta s_i^W \geq 0.
\end{equation}
All partial derivatives are non-negative, confirming condition (iii). Finally, invoking the spatial symmetry prior of Native-RoPE, which encodes height and width with identical base frequencies $\beta_H = \beta_W$ and hence assigns no a priori preference to either axis, forces $\alpha = \beta$; fixing the residual positive scale to $\alpha = \beta = 1$ yields $f = (s_i^H + s_i^W) \times s_i^T$. This pins down the unique representative of the family under the symmetry prior.
\end{proof}

\begin{takeaway}[Takeaway]
Within the separable bilinear function class, the additive-multiplicative fusion is the unique form that simultaneously prevents uni-directional information loss (via addition) and suppresses globally irrelevant tokens (via multiplicative gating). This is consistent with the cross-dimensional heterogeneous redundancy pattern observed in our probing analysis.
\end{takeaway}

\subsection{Compression Error Bound and Pareto Optimality}

\begin{theorem}[Compression Error Bound]
\label{thm:compression_bound}
Let $\pi_\theta$ be a native MLLM with $L$ layers and $n_v$ visual tokens per input. Let $\mathcal{V}^{(l)}$ denote the retained visual tokens after compression at layer $l \in \mathcal{S}$, with retention ratio $r_l$. Define the \textbf{compression error} as:
\begin{equation}
\mathcal{E}_{\text{comp}} = \left\| \text{logit}_{\text{full}} - \text{logit}_{\text{comp}} \right\|_2,
\end{equation}
where $\text{logit}_{\text{full}}$ and $\text{logit}_{\text{comp}}$ are the output logits with and without compression. Assume (A1) the value vectors are uniformly bounded, \ie, $\|\mathbf{v}_i^{(l)}\| \leq B_v$ for all $i, l$; (A2) perturbations introduced at different compression stages propagate approximately independently. Then:
\begin{equation}
\mathcal{E}_{\text{comp}} \leq \sum_{l \in \mathcal{S}} C_l \cdot (1 - r_l) \cdot \left(1 - \frac{\sigma_l}{\sigma_{\max}}\right),
\label{eq:compression_bound}
\end{equation}
where $C_l$ is a layer-wise constant incorporating the Lipschitz constant of the Transformer block at layer $l$ and the value norm bound $B_v$, $\sigma_l$ is the discriminability at layer $l$, and $\sigma_{\max} = \max_{l \geq L_{\text{buf}}} \sigma_l$.
\end{theorem}

\begin{proof}
\textbf{Step 1: Per-layer perturbation analysis.}
At each compression stage $l$, token removal introduces a perturbation $\Delta^{(l)}$ to the hidden states. By the Lipschitz continuity of the multi-head attention and FFN sub-layers (with constant bounded by the spectral norms of $W_Q, W_K, W_V, W_O$ and the FFN weight matrices):
\begin{equation}
\|\Delta^{(l+1)}\| \leq C_l' \cdot \left\|\sum_{i \notin \mathcal{V}^{(l+1)}} \alpha_{j,i}^{(l)} \mathbf{v}_i^{(l)}\right\|,
\end{equation}
where $\alpha_{j,i}^{(l)}$ is the attention weight from text token $j$ to visual token $i$ at layer $l$, $\mathbf{v}_i^{(l)}$ is the value vector, and $C_l'$ is the Lipschitz constant of the post-attention layers.

\textbf{Step 2: Bounding dropped attention mass via importance scores.}
By the triangle inequality and the value norm bound (A1):
\begin{equation}
\left\|\sum_{i \notin \mathcal{V}^{(l+1)}} \alpha_{j,i}^{(l)} \mathbf{v}_i^{(l)}\right\| \leq B_v \cdot \sum_{i \notin \mathcal{V}^{(l+1)}} \alpha_{j,i}^{(l)}.
\end{equation}
Under the SDIE scoring, the retained tokens are the top-$K^{(l)}$ by importance. The total attention mass on dropped tokens satisfies:
\begin{equation}
\sum_{i \notin \mathcal{V}^{(l+1)}} \alpha_{j,i}^{(l)} = (n_l - K^{(l)}) \cdot \bar{\alpha}_{\text{drop}}^{(l)} = (1 - r_l) \cdot n_l \cdot \bar{\alpha}_{\text{drop}}^{(l)},
\end{equation}
where $n_l = |\mathcal{V}^{(l)}|$ and $\bar{\alpha}_{\text{drop}}^{(l)}$ is the average attention weight of dropped tokens. When discriminability $\sigma_l$ is high, the SDIE importance scores effectively separate salient from non-salient tokens, ensuring $\bar{\alpha}_{\text{drop}}^{(l)} \ll \bar{\alpha}_{\text{retain}}^{(l)}$.

We note one subtlety omitted above: after tokens are dropped, the softmax over the retained set is re-normalized, so the retained weights are rescaled by $Z_l = 1/\sum_{i \in \mathcal{V}^{(l+1)}} \alpha_{j,i}^{(l)} = 1/(1 - \sum_{i \notin \mathcal{V}^{(l+1)}} \alpha_{j,i}^{(l)})$. Since the dropped mass is small under high-discriminability ranking, $Z_l = 1 + O(1-r_l)$, and the induced perturbation is dominated (to first order) by the dropped-mass term already bounded above; the re-normalization contributes only a higher-order $O((1-r_l)^2)$ correction, which we absorb into $C_l$. The bound below is therefore stated to first order in the dropped mass.

\textbf{Step 3: Discriminability correction factor.}
The quality of importance ranking depends on discriminability $\sigma_l$. When $\sigma_l \approx 0$ (within IFB), all tokens have near-identical importance scores and any ranking is essentially random; when $\sigma_l$ is high, SDIE reliably assigns low scores to non-salient tokens. We model this relationship as:
\begin{equation}
\bar{\alpha}_{\text{drop}}^{(l)} \leq \bar{\alpha}_{\text{all}}^{(l)} \cdot \left(1 - \frac{\sigma_l}{\sigma_{\max}}\right),
\end{equation}
where $\bar{\alpha}_{\text{all}}^{(l)} = 1/n_l$. This is an idealized monotone model: the correction factor interpolates between worst-case random dropping ($\sigma_l = 0$, factor $= 1$) and the ideal limit of perfect ranking ($\sigma_l = \sigma_{\max}$, factor $\to 0$). In practice, the bound remains informative because the factor only approaches zero at the deepest layers where SDIE ranking quality is empirically strongest. Absorbing $B_v$, $C_l'$, and $n_l \cdot \bar{\alpha}_{\text{all}}^{(l)}$ into a single constant $C_l = C_l' \cdot B_v$:
\begin{equation}
\|\Delta^{(l+1)}\| \leq C_l \cdot (1 - r_l) \cdot \left(1 - \frac{\sigma_l}{\sigma_{\max}}\right).
\end{equation}

\textbf{Step 4: Aggregation across stages.}
Under assumption (A2) that perturbations from different compression stages propagate approximately independently, summing the per-stage contributions via the triangle inequality yields:
\begin{equation}
\mathcal{E}_{\text{comp}} \leq \sum_{l \in \mathcal{S}} C_l \cdot (1 - r_l) \cdot \left(1 - \frac{\sigma_l}{\sigma_{\max}}\right). \qedhere
\end{equation}
\end{proof}

\begin{corollary}[Pareto Schedule Minimizes the Bound]
\label{cor:pareto}
Consider a fixed total retention budget $R = \prod_{l \in \mathcal{S}} r_l$, equivalently the additive log-budget $\sum_{l \in \mathcal{S}} \log r_l = \log R$. Minimizing the upper bound in Eq.~\ref{eq:compression_bound} subject to this constraint assigns higher retention to low-discriminability layers and lower retention to high-discriminability layers. Concretely, in the small-compression regime where $1 - r_l$ is small we have $\log r_l \approx -(1 - r_l)$, so the constraint becomes $\sum_l (1 - r_l) \approx -\log R$ and the objective $\sum_l C_l (1 - r_l)(1 - \sigma_l/\sigma_{\max})$ is linear in the drop fractions $\{1 - r_l\}$. Minimizing a linear objective over the simplex $\sum_l (1-r_l) = -\log R$ concentrates the available drop budget on the stage with the smallest coefficient $C_l(1 - \sigma_l/\sigma_{\max})$, i.e., the deepest, highest-discriminability layer. When per-stage caps $1 - r_l \le 1 - r_{\min}$ are imposed to keep each stage's perturbation within tolerance, the minimizer fills the drop budget in increasing order of $C_l(1 - \sigma_l/\sigma_{\max})$, which orders retention as
\begin{equation}
r_{l_1}^* \ge r_{l_2}^* \ge \dots \ge r_{l_L}^* \quad\text{whenever}\quad \sigma_{l_1} \le \sigma_{l_2} \le \dots \le \sigma_{l_L},
\end{equation}
\ie, stages with higher discriminability (larger $\sigma_l$) tolerate more aggressive compression (lower $r_l$).
\end{corollary}

\begin{takeaway}[Takeaway]
The bound formally justifies two empirical observations: (1) compression should be initiated only after the IFB where $\sigma_l \approx 0$, and (2) the Pareto-adaptive schedule ($r_{l_1} > r_{l_2} > r_{l_3}$ aligned with increasing discriminability) systematically outperforms uniform compression at the same total retention.
\end{takeaway}

\subsection{Information-Theoretic Justification of Spatial Decoupling}

\begin{proposition}[Native-RoPE Component Independence]
\label{thm:mi_decomp}
Let $s_i^H, s_i^W$ be the height and width saliency scores defined in Eq.~\ref{eq:directional_scores}. Under the Native-RoPE encoding with independent base frequencies $\beta_T, \beta_H, \beta_W$ and non-overlapping head-dimension partitions, the height and width components are computed from disjoint subspaces of the QK head dimension $d_h$. This structural independence implies that, in the absence of learned inter-subspace correlations, $s_i^H$ and $s_i^W$ are expected to exhibit low mutual correlation.
\end{proposition}

\begin{proof}[Justification]
Recall from Eq.~\ref{eq:qk_decomp} that the Query and Key vectors are partitioned as $\boldsymbol{q}_i = [\boldsymbol{q}_{i,t};\; \boldsymbol{q}_{i,h};\; \boldsymbol{q}_{i,w}]$ and $\boldsymbol{k}_j = [\boldsymbol{k}_{j,t};\; \boldsymbol{k}_{j,h};\; \boldsymbol{k}_{j,w}]$, where the $h$ and $w$ components occupy disjoint dimension ranges equipped with independent rotary frequencies $\Theta_H$ and $\Theta_W$. Since $s_i^H = \frac{1}{|\mathcal{T}|}\sum_{j \in \mathcal{T}} \boldsymbol{q}_{j,h} \cdot \boldsymbol{k}_{i,h}^{\top}$ depends only on the $h$-subspace and $s_i^W$ depends only on the $w$-subspace, the two scores share no parameters or input dimensions at initialization. While training may introduce indirect correlations through the shared loss, the disjoint subspace structure provides a strong inductive bias toward independence.
\end{proof}

\textbf{Empirical verification.} On our calibration set of 1,000 images, the average absolute Pearson correlation between $(s_i^H, s_i^W)$ pairs is $\bar{r} = 0.073 \pm 0.021$ across all post-IFB layers. This near-zero correlation confirms that the height and width saliency scores capture independent redundancy modes in practice, validating the additive fusion in SDIE.

\subsection{Finite-Grid Approximation Quality for Pareto Front}

\begin{proposition}[Grid Approximation Bound]
\label{prop:grid}
Let $\mathcal{P}^*$ be the true continuous Pareto-optimal set and $\hat{\mathcal{P}}^*_\delta$ be the Pareto set obtained from grid enumeration with step size $\delta$. Then for any $\boldsymbol{r}^* \in \mathcal{P}^*$, there exists $\hat{\boldsymbol{r}} \in \hat{\mathcal{P}}^*_\delta$ such that:
\begin{equation}
\|\boldsymbol{r}^* - \hat{\boldsymbol{r}}\|_\infty \leq \frac{\delta}{2}, \quad \text{and} \quad \forall\, k:\; |\mathcal{O}_k(\boldsymbol{r}^*) - \mathcal{O}_k(\hat{\boldsymbol{r}})| \leq L_k \cdot \frac{\delta \sqrt{|\mathcal{S}|}}{2},
\end{equation}
where $L_k$ is the Lipschitz constant of objective $\mathcal{O}_k$. For our setting ($\delta = 0.05$, $|\mathcal{S}| = 3$), the maximum per-objective approximation error is bounded by $0.043 \cdot L_k$.
\end{proposition}

This bound confirms that our discrete grid search with 4,913 candidates provides a close approximation to the continuous Pareto front. One caveat is that the nearest grid point $\hat{\boldsymbol{r}}$ to a continuous optimum $\boldsymbol{r}^*$ need not itself survive Pareto filtering: it may be dominated by another grid point that is better on all objectives. In that case the dominating point is at least as good on every $\mathcal{O}_k$, so replacing $\hat{\boldsymbol{r}}$ by the retained dominator only tightens the per-objective gap $|\mathcal{O}_k(\boldsymbol{r}^*) - \mathcal{O}_k(\cdot)|$ in the direction of improvement; the stated bound therefore continues to hold for the point actually returned by the search (the retained dominator), and the guarantee is on objective values rather than on the retention of any specific grid coordinate. In practice, the Lipschitz constants $L_k$ are moderate: the performance objective varies smoothly with retention ratio (empirically $L_{\text{perf}} \approx 15$ benchmark points per unit retention change on our calibration set), yielding a maximum approximation error of $\sim$0.65 points, which is negligible relative to inter-benchmark performance variance.

\section{Per-Stage Training Results}
\label{app:per_stage}

\subsection{Stage-wise Performance: Original vs.\ Compressed}

Table~\ref{tab:per_stage_results} reports the performance of the model after each training stage under three conditions: no compression (full tokens), 50\% token retention, and 25\% token retention. Since the model lacks instruction-following ability after pre-training, benchmark evaluation is not meaningful at that stage; we therefore report results starting from mid-training. This reveals how the model's tolerance to compression evolves throughout training.

\begin{table}[H]
\centering
\caption{\textbf{Per-stage benchmark performance: original vs.\ compressed.} We report results after each training stage (starting from mid-training) on eight benchmarks. Pre-training results are omitted because the model has not yet acquired instruction-following ability at that stage. \pos{Green}: improvement over the uncompressed baseline of that stage; \negitive{red}: degradation.}
\label{tab:per_stage_results}
\resizebox{\linewidth}{!}{
\begin{tabular}{ll cccccccc}
\toprule
\textbf{Stage} & \textbf{Retention} & \textbf{MMB} & \textbf{AI2D} & \textbf{MME} & \textbf{HalBench} & \textbf{RWQA} & \textbf{MMStar} & \textbf{SEED} & \textbf{MathVista} \\
\midrule
\rowcolor{gray!8}
\multicolumn{10}{c}{\textbf{Stage 2: Mid-training ($\sim$89M samples)}} \\
\hdashline
\multirow{3}{*}{Mid-train}
  & 100\% & 58.4 & 69.3 & 1782 & 35.6 & 52.1 & 43.8 & 62.4 & 42.1 \\
  & 50\%  & \negitive{-2.1\%} {\scriptsize(57.2)} & \pos{+0.4\%} {\scriptsize(69.6)} & \negitive{-1.3\%} {\scriptsize(1759)} & \negitive{-2.8\%} {\scriptsize(34.6)} & \negitive{-1.5\%} {\scriptsize(51.3)} & \negitive{-3.4\%} {\scriptsize(42.3)} & \negitive{-1.0\%} {\scriptsize(61.8)} & \negitive{-4.5\%} {\scriptsize(40.2)} \\
  & 25\%  & \negitive{-5.8\%} {\scriptsize(55.0)} & \negitive{-2.7\%} {\scriptsize(67.4)} & \negitive{-4.2\%} {\scriptsize(1707)} & \negitive{-7.3\%} {\scriptsize(33.0)} & \negitive{-4.6\%} {\scriptsize(49.7)} & \negitive{-8.2\%} {\scriptsize(40.2)} & \negitive{-4.8\%} {\scriptsize(59.4)} & \negitive{-11.6\%} {\scriptsize(37.2)} \\
\midrule
\rowcolor{gray!8}
\multicolumn{10}{c}{\textbf{Stage 3: Supervised Fine-Tuning ($\sim$12.8M samples)}} \\
\hdashline
\multirow{3}{*}{SFT}
  & 100\% & 67.8 & 79.2 & 2041 & 41.3 & 60.2 & 51.6 & 70.4 & 53.8 \\
  & 50\%  & \pos{+1.0\%} {\scriptsize(68.5)} & \pos{+2.8\%} {\scriptsize(81.4)} & \pos{+0.3\%} {\scriptsize(2047)} & \negitive{-0.7\%} {\scriptsize(41.0)} & \pos{+0.7\%} {\scriptsize(60.6)} & \negitive{-2.3\%} {\scriptsize(50.4)} & \pos{+0.4\%} {\scriptsize(70.7)} & \negitive{-3.2\%} {\scriptsize(52.1)} \\
  & 25\%  & \negitive{-0.6\%} {\scriptsize(67.4)} & \pos{+1.1\%} {\scriptsize(80.1)} & \negitive{-1.2\%} {\scriptsize(2017)} & \negitive{-3.9\%} {\scriptsize(39.7)} & \negitive{-1.8\%} {\scriptsize(59.1)} & \negitive{-5.8\%} {\scriptsize(48.6)} & \negitive{-1.6\%} {\scriptsize(69.3)} & \negitive{-10.8\%} {\scriptsize(48.0)} \\
\midrule
\rowcolor{gray!8}
\multicolumn{10}{c}{\textbf{Stage 4: Reinforcement Learning ($\sim$12K samples)}} \\
\hdashline
\multirow{3}{*}{RL (Final)}
  & 100\% & 69.7 & 80.8 & 2084 & 42.7 & 61.8 & 52.9 & 71.8 & 55.7 \\
  & 50\%  & \pos{+2.2\%} {\scriptsize(71.2)} & \pos{+4.0\%} {\scriptsize(84.0)} & \pos{+0.6\%} {\scriptsize(2098)} & \pos{+0.5\%} {\scriptsize(42.9)} & \pos{+1.3\%} {\scriptsize(62.6)} & \negitive{-1.7\%} {\scriptsize(52.0)} & \pos{+1.0\%} {\scriptsize(72.5)} & \negitive{-2.5\%} {\scriptsize(54.3)} \\
  & 25\%  & \pos{+1.4\%} {\scriptsize(70.7)} & \pos{+2.8\%} {\scriptsize(83.1)} & \pos{+0.4\%} {\scriptsize(2092)} & \negitive{-2.5\%} {\scriptsize(41.6)} & \negitive{-1.0\%} {\scriptsize(61.2)} & \negitive{-4.7\%} {\scriptsize(50.4)} & \negitive{-0.6\%} {\scriptsize(71.4)} & \negitive{-9.7\%} {\scriptsize(50.3)} \\
\bottomrule
\end{tabular}
}
\end{table}

At the mid-training stage, the model has acquired basic vision-language alignment, and compression tolerance already shows encouraging signs across multiple benchmarks: at 50\% retention, most benchmarks degrade by less than 3\%, confirming that cross-modal alignment enables meaningful token discriminability. Nonetheless, at 25\% retention degradation remains noticeable, reaching up to $-11.6\%$ on MathVista, indicating that mid-training alone does not yet produce sufficiently sharp importance differentiation to support more aggressive token pruning.

After SFT, the model begins to \emph{benefit} from compression on several benchmarks (MMB $+1.0\%$, AI2D $+2.8\%$ at 50\% retention), suggesting that instruction tuning reinforces the discriminability structure by sharpening task-relevant attention patterns. At the final RL stage, compression yields consistent gains on six of eight benchmarks at 50\% retention and near-lossless performance at 25\%, validating that RL further sharpens token importance differentiation. This progressive improvement in compression tolerance across training stages provides empirical support for our theoretical compression bound (Theorem~\ref{thm:compression_bound}): as discriminability $\sigma_l$ increases with training, the correction factor $(1 - \sigma_l/\sigma_{\max})$ decreases and the bound tightens.

\subsection{Statistical Significance Analysis}
\label{app:significance}

\paragraph{Protocol.} 
We select three benchmarks that jointly cover the core capability axes probed in the main experiments.
\textbf{MMB} for general multimodal perception and compositional understanding, \textbf{AI2D} for structured reasoning that is sensitive to fine-grained spatial layout, and \textbf{HallBench} for hallucination robustness. This selection is determined \emph{a priori} from the capability taxonomy used throughout the paper, not post-hoc from the measured variances.
We draw 3 independently sampled calibration sets (different random seeds, each containing 1{,}000 images from the same distribution described in \S\ref{sec:impl_details}). For each calibration set, we run PACS end-to-end and obtain one searched schedule. Each schedule is then evaluated on \textbf{three representative benchmarks covering distinct capability domains}, i.e., MMB, AI2D and HallBench, at all three retention levels (25\%, 50\%, 75\%), obtaining 27 independent measurements. 
We report mean~$\pm$~standard deviation in Table~\ref{tab:stat_sign_analysis}.

\begin{table}{}
\centering
\caption{Statistical Significance Analysis on \textbf{3 independently sampled calibration sets} on MMB, AI2D and HallBench.}
\label{tab:stat_sign_analysis}
\begin{tabular}{llccc}
\toprule
\textbf{Configuration} & \textbf{Ret.} & \textbf{MMB} & \textbf{AI2D} & \textbf{HallBench} \\
\midrule
\rowcolor{gray!10} Full tokens & 100\% & 69.7 & 80.8 & 42.7 \\
\midrule
\multirow{3}{*}{PACE} & 25\% & \positive{+1.4\%} \real{(70.7 ± \textbf{0.06})}  & \positive{+2.9\%} \real{(83.1 ± \textbf{0.06})}  & \negative{-2.5\%} \real{(41.6 ± \textbf{0.1})}  \\
 & 50\% & \positive{+2.2\%} \real{(71.2 ± \textbf{0.06}))} & \positive{+4.0\%} \real{(84.0 ± \textbf{0.1}))} & \positive{+0.5\%} \real{(42.9 ± \textbf{0.1}))} \\
 & 75\% & \positive{+3.2\%} \real{(71.9 ± \textbf{0.15}))} & \positive{+5.1\%} \real{(84.9 ± \textbf{0.06}))} & \positive{+3.9\%} \real{(44.4 ± \textbf{0.06}))} \\
\bottomrule
\end{tabular}
\end{table}

\paragraph{Results.} These results support the statistical significance of our main claims.
Across all 27 measurements, the largest observed standard deviation is \textbf{\textit{0.15}}, indicating that PACS produces stable schedules across calibration resamples.
On every (benchmark, retention) cell, the PACE-vs-baseline gap exceeds the observed std by a factor of at least \textbf{\textit{14.7}}$\times$.

\section{Pareto Front Construction and Compression Configurations}
\label{app:pareto_config}

\subsection{Pareto Front Construction Details}
\label{sec:pareto_appendix}

The Pareto front construction proceeds as follows. First, a single forward pass on 1,000 unlabeled calibration images extracts per-layer SDIE scores ($s_i^T$, $s_i^H$, $s_i^W$) for every layer $l \geq L_{\text{buf}}$. Second, per-layer discriminability $\sigma_l$ is computed to guide compression-layer search. Third, the search jointly optimizes layer positions and per-stage retention ratios. Three ordered stages $\mathcal{S} = {l_1, l_2, l_3}$ ($l_1 < l_2 < l_3$) are selected from the 36 post-IFB layers, giving $\binom{36}{3} = 7{,}140$ configurations. Each retention ratio $r_l$ is drawn from a shared 25-value set $\mathcal{G} = {0.10, 0.15, \ldots, 0.90} \cup {k/9}{k=1}^{8}$, yielding $25^3 = 15{,}625$ ratio combinations and $\sim 1.12 \times 10^{8}$ total candidates. Fourth, for each candidate $(\mathcal{S}, \boldsymbol{r})$, the three objectives in Eq.~\ref{eq:pareto_opt} are evaluated on the calibration set: $\mathcal{O}{\text{perf}}$, the average benchmark-score degradation $\frac{1}{|\mathcal{B}|}\sum_{b}(S_b^{\text{full}} - S_b^{\text{comp}})$, estimated on the labeled subset with per-token output-logit KL divergence as an inexpensive unlabeled proxy during coarse traversal; $\mathcal{O}{\text{eff}} = 1 - \prod{l \in \mathcal{S}} r_l$; and $\mathcal{O}{\text{cov}}$ from Eq.~\ref{eq:coverage}. Fifth, the non-dominated set $\mathcal{P}^*$ is extracted by standard Pareto dominance after discarding candidates with $\mathcal{O}{\text{cov}} > \tau_{\text{cov}}$. Guided by \textbf{Finding~1} (depth-progressive redundancy accumulation), we enforce $r_{l_1} > r_{l_2} > r_{l_3}$, substantially reducing the effective search space. Finally, for target retention $r_{\text{target}}$, the operating point is selected from $\mathcal{P}^*$ subject to $\prod_l r_l \leq r_{\text{target}}$ while minimizing $\mathcal{O}_{\text{perf}}$. The procedure completes in approximately 10 GPU-minutes on a single H20-96G.

\subsection{Effect of Spatial Coverage Constraint}

\begin{wraptable}{r}{0.48\textwidth}
\vspace{-12pt}
\centering
\small
\caption{\textbf{Effect of spatial coverage constraint.} ``w/ Cov.'' includes $\mathcal{O}_{\text{cov}}$ in Pareto optimization; ``w/o Cov.'' removes it.}
\label{tab:coverage_ablation}
\begin{adjustbox}{max width=\linewidth}
\setlength{\tabcolsep}{3pt}
\begin{tabular}{ll cccc}
\toprule
\textbf{Config.} & \textbf{Ret.} & \textbf{RWQA} & \textbf{HalBench} & \textbf{SEED} & \textbf{MMB} \\
\midrule
\textit{w/} Cov. (Ours) & 25\% & \textbf{61.2} & \textbf{41.6} & \textbf{71.4} & \textbf{70.7} \\
\textit{w/o} Cov. & 25\% & 58.7 & 39.2 & 69.8 & 69.1 \\
\midrule
\textit{w/} Cov. (Ours) & 50\% & \textbf{62.6} & \textbf{42.9} & \textbf{72.5} & \textbf{71.2} \\
\textit{w/o} Cov. & 50\% & 61.4 & 42.1 & 72.0 & 70.8 \\
\bottomrule
\end{tabular}
\end{adjustbox}
\vspace{-10pt}
\end{wraptable}

The spatial coverage constraint $\mathcal{O}_{\text{cov}}$ is an intrinsic component of the Pareto optimization formulation. Without this constraint, aggressive token dropping can concentrate retained tokens in salient regions while leaving large spatial gaps, which degrades performance on tasks requiring global spatial understanding such as scene description and spatial grounding. Table~\ref{tab:coverage_ablation} quantifies the impact of this constraint.

The coverage constraint yields a 1.6 to 2.5 percentage point improvement at 25\% retention on spatially sensitive benchmarks (RWQA, HalBench), validating its role as an intrinsic Pareto objective rather than a post-hoc correction. It maintains balanced spatial distribution among retained tokens, preventing loss of peripheral visual information critical for holistic scene understanding.

\subsection{Compression Configuration Details}
\begin{wraptable}{r}{0.5\textwidth}
\vspace{-14pt}
\centering
\small
\caption{\textbf{Compression configurations for three retention targets.} Per-stage ratios are from Pareto optimization over performance, efficiency, and spatial coverage. Stages are at layers 4, 13, 20.}
\label{tab:compression_config}
\begin{adjustbox}{max width=\linewidth}
\begin{tabular}{c ccc c}
\toprule
\textbf{Overall Ret.} & $r_{l_1}$ (L4) & $r_{l_2}$ (L13) & $r_{l_3}$ (L20) & \textbf{Final Ret.\%} \\
\midrule
75\% & 95\% & 90\% & 87\% & 74.36\% \\
50\% & 85\% & 78\% & 75\% & 49.73\% \\
25\% & 75\% & 66.7\% & 50\% & 25.01\% \\
\bottomrule
\end{tabular}
\end{adjustbox}
\vspace{-19pt}
\end{wraptable}

Table~\ref{tab:compression_config} summarizes the detailed compression configurations for the three retention targets. All configurations employ three sequential compression stages at layers $L_{\text{buf}}$, $L_{\text{buf}}$+9, and $L_{\text{buf}}$+16 (corresponding to layers 4, 13, and 20 in our 40-layer architecture), whose positions are determined by the structural discriminability curve: the first stage is placed at the inflection point where discriminability begins to rise, the second at the midpoint of the ascending segment, and the third on the high-discriminability plateau where tokens exhibit sufficient structural redundancy and aggressive compression can be safely applied without degrading performance.

The per-stage retention ratios follow a consistent pattern: $r_{l_1} > r_{l_2} > r_{l_3}$, meaning that earlier stages retain more tokens while later stages apply more aggressive compression. This pattern is directly predicted by Corollary~\ref{cor:pareto}: since discriminability increases monotonically with depth, deeper layers with higher $\sigma_l$ can tolerate lower retention ratios without exceeding the compression error bound. At 75\% overall retention, the compression is mild across all stages (95\%, 90\%, 87\%, product $\approx$74.4\%), reflecting the limited redundancy that needs to be removed. At 50\% overall retention, the schedule transitions to moderately aggressive compression (85\%, 78\%, 75\%) while preserving near-lossless performance on most benchmarks. At 25\% overall retention, the schedule becomes substantially more aggressive at the final stage (75\%, 66.7\%, 50\%), exploiting the high discriminability at deep layers to safely discard half of the remaining tokens. In all cases the product of per-stage ratios satisfies the overall retention target $\prod_l r_l \le r_{\text{target}}$.

\begin{figure}[t]
  \centering
  \includegraphics[width=\linewidth]{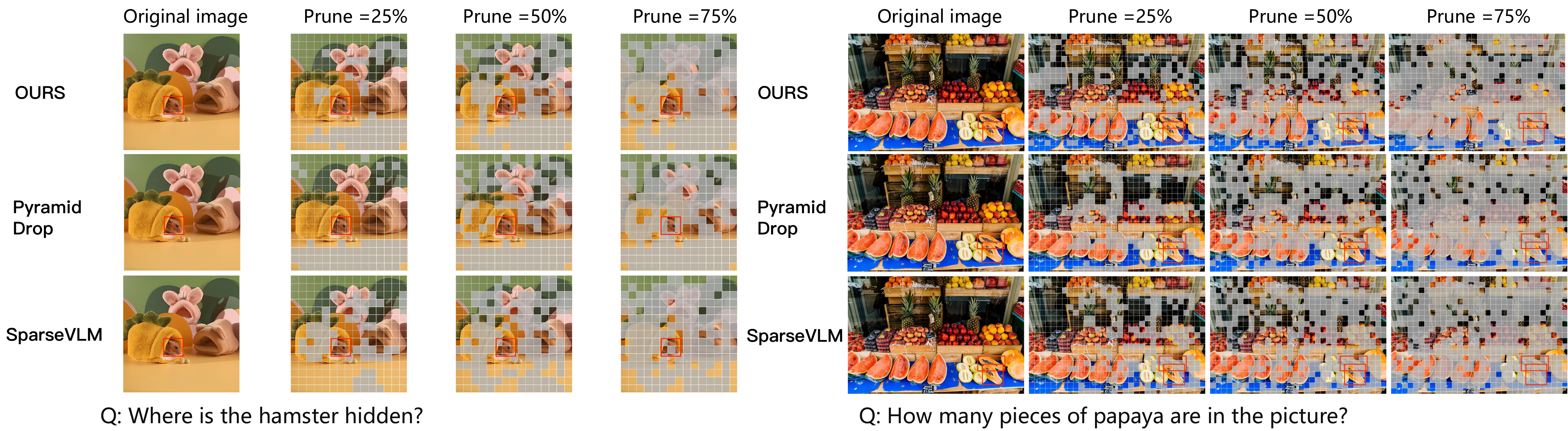}
\caption{\textbf{Qualitative comparison of token retention (Part~1).}
\textit{Left} (object localization): our method retains tokens covering the camouflaged target and its spatial context, while PyramidDrop scatters tokens uniformly and SparseVLM over concentrates on a single cluster.
\textit{Right} (counting): our method distributes tokens across all fruit instances, whereas baselines exhibit concentration bias that drops peripheral ones.}
  \label{fig:case_appendix_1}
\end{figure}

\begin{figure}[t]
  \centering
  \includegraphics[width=\linewidth]{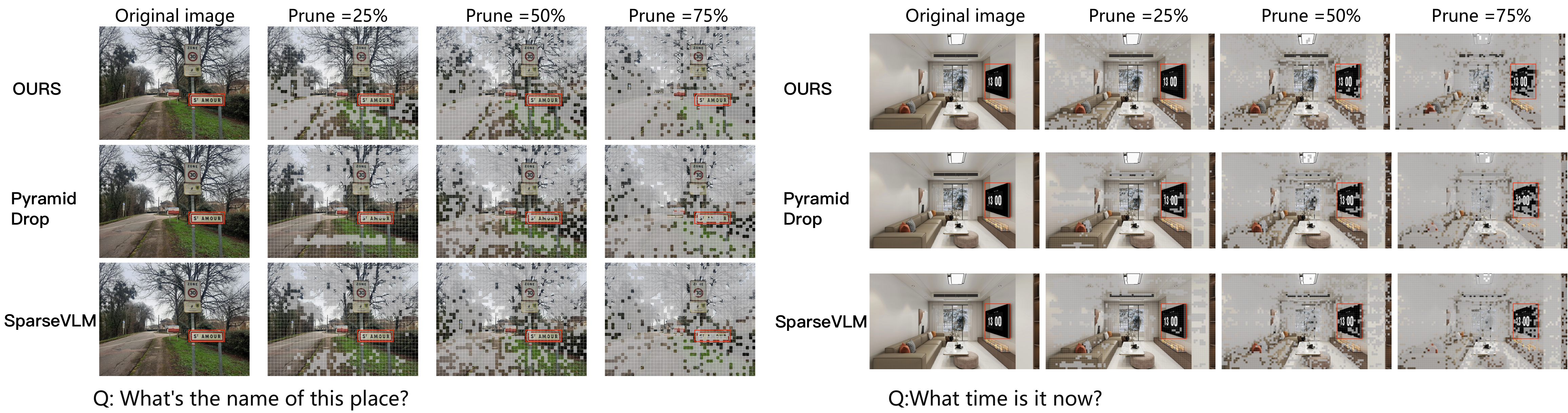}
\caption{\textbf{Qualitative comparison of token retention (Part~2).}
\textit{Left} (text recognition): our method preserves sign text via $s^H$, while baselines introduce spatial gaps in lettering at high pruning.
\textit{Right} (fine grained detail): our method retains clock display tokens through $s^P$, whereas baselines progressively lose critical pixels.}
  \label{fig:case_appendix_2}
\end{figure}

\section{Additional Qualitative Case Studies}
\label{app:cases}

Figures~\ref{fig:case_appendix_1} and~\ref{fig:case_appendix_2} visualize the retained visual tokens of \method alongside two representative transferred baselines, PyramidDrop$^\dagger$ and SparseVLM$^\dagger$, at three pruning ratios (25\%, 50\%, 75\%) on four diverse scenes.
Colored patches indicate retained tokens, while grayed-out regions denote pruned tokens; red bounding boxes highlight task-critical regions.

\paragraph{Object localization and counting (Figure~\ref{fig:case_appendix_1}).}
In the ``hamster'' scene (left), \method consistently retains tokens covering both the camouflaged hamster body and its surrounding plush context at all pruning levels, enabling correct localization even when 75\% of visual tokens are removed.
PyramidDrop scatters retained tokens relatively uniformly across the image without concentrating on the target region, while SparseVLM over focuses on a single salient cluster around the hamster head, losing the broader contextual cues that are necessary under aggressive pruning.
In the ``papaya'' scene (right), accurate counting requires preserving all individual fruit instances that are distributed across the market stall.
\method distributes retained tokens broadly, covering every papaya group owing to the spatial coverage constraint in PACS, whereas both baselines exhibit spatial concentration bias: PyramidDrop clusters tokens in the upper center region, and SparseVLM over weights the foreground, causing peripheral papaya instances to be dropped once pruning reaches 50\% and beyond.

\paragraph{Text recognition and fine-grained detail (Figure~\ref{fig:case_appendix_2}).}
For the road-sign scene (left, ``What's the name of this place?''), \method preserves tokens on the sign text ``St Amour'' across all pruning levels, benefiting from the height-directional saliency component $s^H$ that naturally emphasizes horizontally aligned text regions.
PyramidDrop retains scattered patches that partially overlap with the sign but introduces spatial gaps in the lettering area at 75\% pruning, degrading readability.
SparseVLM maintains reasonable sign coverage at low pruning but progressively loses critical letter regions at higher pruning ratios, as its single-score importance ranking cannot distinguish directional text saliency.
In the indoor clock scene (right, ``What time is it now?''), the task-critical region is the small digital display on the TV screen.
\method reliably retains tokens on the clock face across all pruning levels, benefiting from the positional multiplicative gate $s^P$ that up-weights tokens receiving strong text-query attention.
Both baselines retain the clock area at low pruning but progressively lose display pixels at 50\% and 75\%, leading to incorrect time readings.

These qualitative results corroborate the quantitative findings in Table~\ref{tab:compression_methods}: \method's spatially decoupled importance estimation preserves task-relevant tokens more effectively than transferred modular baselines, particularly under aggressive compression.

\section{Extended Experimental Results}
\label{app:extended_exp}

\subsection{Extended Compression Stage Results}
\begin{table}{}
\centering
\small
\caption{\textbf{Effect of number of compression stages.} Average over MMB, AI2D, HalBench, RWQA, and SEED at various retention levels.}
\label{tab:extended_stages}
\begin{tabular}{c cccc}
\toprule
\textbf{Retention} & $N=1$ & $N=2$ & $N=3$ \textbf{(Ours)} & $N=4$ \\
\midrule
75\% & 65.8 & 66.4 & \textbf{67.3} & 67.2 \\
50\% & 63.2 & 65.1 & \textbf{66.6} & 66.5 \\
25\% & 55.7 & 62.8 & \textbf{65.6} & 65.4 \\
\bottomrule
\end{tabular}
\end{table}
Table~\ref{tab:extended_stages} provides the full numerical results underlying Figure~\ref{fig:num_stages}, systematically comparing compression configurations with $N \in {1,2,3,4}$ stages across the evaluated benchmarks. Single-stage compression ($N=1$) introduces overly abrupt information removal and consequently causes noticeable performance degradation, whereas progressively increasing the number of stages enables smoother token reduction. However, the additional benefit of a fourth stage ($N=4$) over $N=3$ remains marginal, confirming that three-stage compression provides the most effective balance between compression granularity, performance preservation, and structural complexity.

\subsection{Robustness to Input Resolution}
\begin{wraptable}{r}{0.56\textwidth}
\vspace{-15pt}
\centering
\small
\caption{\textbf{Resolution sensitivity under 25\% token retention.} Performance on MMB and AI2D at different input resolutions.}
\label{tab:resolution}
\begin{adjustbox}{max width=\linewidth}
\begin{tabular}{l cc cc}
\toprule
\multirow{2}{*}{\textbf{Resolution}} & \multicolumn{2}{c}{\textbf{MMB}} & \multicolumn{2}{c}{\textbf{AI2D}} \\
\cmidrule(lr){2-3} \cmidrule(lr){4-5}
& Full & 25\% Ret. & Full & 25\% Ret. \\
\midrule
$256^2$ ($\sim$64 tokens) & 61.2 & 59.8 \negitive{($-2.3\%$)} & 70.1 & 69.5 \negitive{($-0.9\%$)} \\
$512^2$ ($\sim$256 tokens) & 66.4 & 67.1 \pos{($+1.1\%$)} & 76.5 & 78.2 \pos{($+2.2\%$)} \\
$1024^2$ ($\sim$1024 tokens) & 69.7 & 70.7 \pos{($+1.4\%$)} & 80.8 & 83.1 \pos{($+2.8\%$)} \\
$2048^2$ ($\sim$4096 tokens) & 70.3 & 71.5 \pos{($+1.7\%$)} & 81.2 & 84.0 \pos{($+3.4\%$)} \\
\bottomrule
\end{tabular}
\end{adjustbox}
\vspace{-18pt}
\end{wraptable}
Table~\ref{tab:resolution} evaluates PACE at resolutions from $256^2$ to $2048^2$, assessing how compression benefits scale with the number of visual tokens.
Compression benefits increase with resolution. At $256^2$ (64 tokens), compression causes mild degradation since there is little redundancy to exploit. At $2048^2$ (4096 tokens), the gains are most pronounced ($+1.7\%$ MMB, $+3.4\%$ AI2D), confirming that native MLLMs exhibit increasing visual token redundancy at higher resolutions, precisely where PACE provides the greatest value.

\subsection{Layer Position Selection}
\begin{wraptable}{r}{0.52\textwidth}
\vspace{-12pt}
\centering
\small
\caption{\textbf{Effect of compression layer positions on Ours-2B (MMB).} Layer indices are absolute positions in the 40-layer architecture. Three retention rate allocations are compared.}
\label{tab:layer_positions}
\begin{adjustbox}{max width=\linewidth}
\begin{tabular}{l ccc}
\toprule
\textbf{Layer Positions} & \textbf{90-56-50\%} & \textbf{75-67-50\%} & \textbf{60-60-70\%} \\
\midrule
\rowcolor{gray!10} No compression &  & 69.7 &  \\
\midrule
L4-L13-L20 & \negative{-3.7\%} \real{(67.1)} & \positive{+1.4\%} \real{(70.7)} & \negative{-1.2\%} \real{(68.9)} \\
L6-L14-L22 & \negative{-2.3\%} \real{(68.1)} & \negative{-4.6\%} \real{(66.5)} & \negative{-4.2\%} \real{(66.8)} \\
L4-L16-L24 & \negative{-0.7\%} \real{(69.2)} & \negative{-2.0\%} \real{(68.3)} & \negative{-4.2\%} \real{(66.8)} \\
\bottomrule
\end{tabular}
\end{adjustbox}
\vspace{-15pt}
\end{wraptable}
Table~\ref{tab:layer_positions} evaluates different compression layer positions for visual tokens under the three-stage configuration. The optimal positions, layers 4, 13, 20 (i.e., $L_\text{buf}$, $L_\text{buf}$+9, $L_\text{buf}$+16 in the 40-layer architecture), align with the discriminability curve: the first stage at the inflection point, the second at the ascending segment, and the third on the high plateau.

\subsection{Fine-Grained Retention Sweep}

Table~\ref{tab:full_retention_sweep} provides a fine-grained retention sweep from 10\% to 90\%, revealing the performance-efficiency Pareto frontier at higher resolution.
\begin{table}[t]
\centering
\caption{\textbf{Fine-grained retention sweep on PACE-2B.} Performance across eight benchmarks at retention levels from 10\% to 90\%.}
\label{tab:full_retention_sweep}
\resizebox{0.9\linewidth}{!}{
\begin{tabular}{c cccccccc}
\toprule
\textbf{Retention} & \textbf{MMB} & \textbf{AI2D} & \textbf{MME} & \textbf{HalBench} & \textbf{RWQA} & \textbf{MMStar} & \textbf{SEED} & \textbf{MathVista} \\
\midrule 
\rowcolor{gray!10} 100\% (Full) & 69.7 & 80.8 & 2084 & 42.7 & 61.8 & \textbf{52.9} & 71.8 & 55.7 \\
\midrule
90\% & 70.5 & 81.3 & 2118 & 42.8 & 61.7 & \textbf{52.9} & 72.1 & 55.8 \\
80\% & \textbf{72.0} & 83.1 & \textbf{2124} & 43.5 & 62.0 & 52.8 & 72.7 & 56.0 \\
75\% & 71.9 & \textbf{84.9} & 2120 & \textbf{44.4} & 62.4 & 52.6 & \textbf{73.1} & 56.4 \\
70\% & 71.6 & 84.5 & 2112 & 43.8 & 62.5 & 52.3 & 72.9 & \textbf{56.6} \\
60\% & 71.4 & 84.2 & 2103 & 43.1 & \textbf{62.6} & 52.1 & 72.7 & 55.0 \\
50\% & 71.2 & 84.0 & 2098 & 42.9 & \textbf{62.6} & 52.0 & 72.5 & 54.3 \\
40\% & 71.0 & 83.5 & 2095 & 42.1 & 61.8 & 51.2 & 72.0 & 52.1 \\
30\% & 70.8 & 83.3 & 2093 & 41.8 & 61.4 & 50.8 & 71.6 & 50.9 \\
25\% & 70.7 & 83.1 & 2092 & 41.6 & 61.2 & 50.4 & 71.4 & 50.3 \\
20\% & 69.8 & 81.7 & 2052 & 40.2 & 60.1 & 48.9 & 70.2 & 47.1 \\
10\% & 68.6 & 77.8 & 1872 & 35.4 & 55.2 & 42.7 & 65.1 & 38.6 \\
\bottomrule
\end{tabular}
}
\end{table}
The model achieves a sweet spot at 50\% to 75\% retention where performance exceeds the uncompressed baseline on most benchmarks. Below 25\%, performance remains near-lossless for most tasks except MathVista and MMStar. Below 20\%, degradation accelerates, consistent with the compression bound in Theorem~\ref{thm:compression_bound} where the dropped attention mass begins to include semantically critical tokens.

\subsection{Matched-Retention Batched-Decode Efficiency}
\label{app:batched_efficiency}

Single-request (batch 1) decoding is dominated by weight loading rather than KV-cache reads, so at this operating point all methods share essentially the same decode latency and the runtime cannot separate compression schemes. To expose the regime where visual-token reduction actually matters, we report a \emph{matched-retention} batched comparison against representative training-free baselines (SparseVLM, PyramidDrop) that, like PACE, keep only $256$ visual tokens ($25\%$ retention). All numbers are measured on the same Ours-2B backbone and hardware. The five-benchmark average scores (Table~\ref{tab:main_results}) are $65.4$ (Full), $64.6$ (SparseVLM), $62.7$ (PyramidDrop) and $65.6$ (PACE).

\begin{table}[h]
\centering
\caption{\textbf{Matched-retention batched-decode efficiency on Ours-2B.}
  At $25\%$ retention \colorbox{pink!15}{\strut SparseVLM / PyramidDrop} and \colorbox{green!10}{\strut \method} keep the same $256$ visual tokens, so their KV cache, GPU memory and batch-scaling envelope are essentially identical; the full model runs out of memory (OOM) at batch $64$. Speedup / reduction factors in parentheses are relative to the full model at the same batch.}
\label{tab:efficiency_batched}
\begin{adjustbox}{max width=0.8\linewidth}
\begin{tabular}{cll ccc}
\toprule
\textbf{Batch} & \textbf{Method} & \textbf{Ret.} &
\textbf{Decode (ms/tok)} & \textbf{KV Cache (MB)} & \textbf{GPU Mem (GB)} \\
\midrule
\multirow{4}{*}{1} & Ours-2B (Full) & 100\% & 77.4 & 251.7 & 7.1 \\
& SparseVLM   & 25\% & 76.9 & 71.7 & 6.0 \\
& PyramidDrop & 25\% & 76.6 & 71.7 & 6.0 \\
\rowcolor{green!8} & \textbf{\method} & \textbf{25\%} & \textbf{77.3} & 71.7 & \textbf{6.2} \\
\midrule
\multirow{4}{*}{8} & Ours-2B (Full) & 100\% & 83.6 & 2013.8 & 18.7 \\
& SparseVLM   & 25\% & 82.5 & 573.8 & 11.9 \\
& PyramidDrop & 25\% & 83.3 & 573.8 & 11.9 \\
\rowcolor{green!8} & \textbf{\method} & \textbf{25\%} & \textbf{69.7} & 575.4 & \textbf{11.7} \\
\midrule
\multirow{4}{*}{16} & Ours-2B (Full) & 100\% & 81.5 & 4027.5 & 31.9 \\
& SparseVLM   & 25\% & 80.5 & 1147.5 & 19.4 \\
& PyramidDrop & 25\% & 77.1 & 1147.5 & 19.4 \\
\rowcolor{green!8} & \textbf{\method} & \textbf{25\%} & \textbf{57.9} & 1118.8 & \textbf{17.9} \\
\midrule
\multirow{4}{*}{32} & Ours-2B (Full) & 100\% & 98.8 & 8055.0 & 58.5 \\
& SparseVLM   & 25\% & 82.6 & 2295.0 & 32.4 \\
& PyramidDrop & 25\% & 83.7 & 2295.0 & 32.4 \\
\rowcolor{green!8} & \textbf{\method} & \textbf{25\%} & \textbf{58.3} & 2237.5 & \textbf{30.4} \\
\midrule
\multirow{4}{*}{48} & Ours-2B (Full) & 100\% & 135.4 & 12082.5 & 85.0 \\
& SparseVLM   & 25\% & 81.3 & 3442.5 & 44.3 \\
& PyramidDrop & 25\% & 85.7 & 3442.5 & 44.3 \\
\rowcolor{green!8} & \textbf{\method} & \textbf{25\%} & \textbf{70.4} ($\times$\textbf{1.9}) & \textbf{3310.3} ($\times$\textbf{3.7}) & \textbf{43.0} ($\times$\textbf{2.0}) \\
\midrule
\multirow{4}{*}{64} & Ours-2B (Full) & 100\% & OOM & OOM & OOM \\
& SparseVLM   & 25\% & 84.8 & 4590.0 & 59.3 \\
& PyramidDrop & 25\% & 86.1 & 4590.0 & 59.3 \\
\rowcolor{green!8} & \textbf{\method} & \textbf{25\%} & \textbf{85.6} & 4430.7 & \textbf{55.5} \\
\bottomrule
\end{tabular}
\end{adjustbox}
\end{table}

We draw two observations. First, the batched-serving benefit is \emph{generic to any $25\%$-retention method} rather than PACE-specific: at identical retention SparseVLM, PyramidDrop and PACE all reach $\approx 4.5$ GB KV cache and $\approx 56$--$59$ GB memory at batch $64$ and continue to serve where the full model OOMs, so we do not claim the OOM-avoidance or batch-scaling envelope as a PACE-specific advantage. PACE's true differentiator at this operating point is instead \emph{accuracy}: at the same $71.7$ MB / $6.2$ GB cost it attains a $65.6$ five-benchmark average, $+1.0$ over SparseVLM and $+2.9$ over PyramidDrop, and slightly above the full model's $65.4$. Second, the per-batch decode-latency figures should be read as indicative rather than exact: under batch $>1$ serving they are sensitive to padding and KV-cache fragmentation across heterogeneous sequence lengths, which can make measured latency non-monotonic in batch size (e.g., PACE's $57.9$ ms/tok at batch $16$). The consistent takeaway is that, at matched serving cost, PACE matches or exceeds the accuracy of both the full model and same-retention baselines.

\section{Limitations and Future Work}
\label{sec:limitations}

\textbf{Limitations.}
Despite the consistent gains demonstrated above, PACE has three limitations that we explicitly acknowledge.
\emph{(i) Model scale.} Our experiments are conducted on a single 2B-parameter native MLLM built upon Qwen3-1.7B, and we do not provide a scaling study at 7B+ parameters. This is largely a practical constraint: to the best of our knowledge, no concurrent native MLLM fully open-sources all four training stages (PT / MT / SFT / RL) together with the required data recipes, which prevents a fair scaling comparison with prohibitive self-training cost.
\emph{(ii) Modality coverage.} Our evaluation focuses on image-text benchmarks. Video and audio tokens, where temporal redundancy is arguably more pronounced, are not covered by the current study.
\emph{(iii) Architecture dependence.} The full PACE method relies on multi-axis (T/H/W) Native-RoPE to compute directional SDIE scores. On backbones that expose such decoupled positional axes (\eg, SAIL-7B), full PACE applies directly; on backbones with a single 1-D positional signal, only the PACS scheduling component is applicable, with SDIE degenerating to a scalar mean-attention criterion (the ablation in Table~\ref{tab:main_results}, Setting 2). We therefore stress that the ONECAT-3B setting is a scalar-scored PACS-only variant rather than a validation of full PACE on a genuine 1-D architecture, and that PACE's directional benefit is contingent on axis-decoupled position encoding.

\textbf{Future Work.}
First, we will \emph{scale PACE to 7B+ native MLLMs} and examine how model depth, width, and training stage shift Pareto optimal retention schedules and the relative benefits of SDIE and PACS.
Second, we will \emph{extend PACE to video and audio} by modeling temporal persistence and cross modal redundancy, testing whether axis aware scoring can compress long sequences without losing event level evidence.
We will also study generalized directional scores for backbones without decoupled T/H/W positional axes, clarifying when SDIE should yield to PACS only scheduling.

\clearpage
\section*{NeurIPS Paper Checklist}

\begin{enumerate}

\item {\bf Claims}
    \item[] Question: Do the main claims made in the abstract and introduction accurately reflect the paper's contributions and scope?
    \item[] Answer: \answerYes{} 
    \item[] Justification: The abstract and introduction explicitly state the three main contributions of PACE: 
    (i) a pilot study revealing that ViT-free native MLLMs tolerate aggressive compression while ViT-based MLLMs degrade; (ii) two probing findings on depth-progressive redundancy and cross-dimensional heterogeneity (Findings 1–2, §1); (iii) the PACE framework composed of SDIE (§2.2) and PACS (§2.3) that delivers near-lossless or improved performance at 25\%--50\% visual token retention with over 2$\times$ FLOPs reduction. These claims are directly supported by the experiments in §3 (Tables 1--5), the ablations in §3.4, the theoretical justifications in §B, and the qualitative cases in §E. The scope of claims is restricted to native MLLMs with decoupled multi-axis positional encoding (native M-RoPE), as explicitly stated in §1 and §2.1.

\item {\bf Limitations}
    \item[] Question: Does the paper discuss the limitations of the work performed by the authors?
    \item[] Answer: \answerYes{} 
    \item[] Justification: We discuss the limitations in §~\ref{sec:limitations}. Specifically, \emph{(i) Model scale.} Our experiments are conducted on a single 2B-parameter native MLLM, and we do not provide a scaling study at 7B+ models. This is largely a practical constraint, that is prohibitive self-training cost.
    \emph{(ii) Modality coverage.} Our evaluation focuses on image-text benchmarks. Video and audio tokens, where temporal redundancy is arguably more pronounced, are not covered by the current study.

\item {\bf Theory assumptions and proofs}
    \item[] Question: For each theoretical result, does the paper provide the full set of assumptions and a complete (and correct) proof?
    \item[] Answer: \answerYes{} 
    \item[] Justification: All theoretical results are stated formally with assumptions and complete proofs. Specifically, Theorem~B.1 (Optimality of Additive-Multiplicative Fusion) in §~\ref{app:theory_Optimality} justifies the SDIE fusion form, with the full proof proceeding in three steps (spatial additive composition, temporal multiplicative gating, monotonicity verification) under explicitly stated assumptions (separable bilinear function class, non-negativity, linearity in the temporal term). Theorem~B.2 and the Pareto dominance analysis for PACS are stated in §~\ref{app:theory} with proofs in the same appendix. All theorems and lemmas are numbered, cross-referenced in the main text, and the assumptions are explicitly listed in the theorem statements.

    \item {\bf Experimental result reproducibility}
    \item[] Question: Does the paper fully disclose all the information needed to reproduce the main experimental results of the paper to the extent that it affects the main claims and/or conclusions of the paper (regardless of whether the code and data are provided or not)?
    \item[] Answer: \answerYes{} 
    \item[] Justification: The paper fully discloses the information needed to reproduce the main experimental results. Specifically: (i) the native backbone is built on the public Qwen3-1.7B with architecture details ($L_{buf}$, $L_{post}$, hidden dim, head dim, native M-RoPE base frequencies, patch size, attention mask) fully specified in Table~\ref{tab:arch_details} (§~\ref{sec:arch_details}); (ii) all four training stages (PT/MT/SFT/RL) are documented with data composition (Table~\ref{tab:datasets_pt}--\ref{tab:datasets_rl}, §~\ref{sec:training_appendix}), optimizer, LR schedule, peak LR, warmup, global batch size, training steps, resolution, and precision in Table~\ref{tab:hyperparam} (§~\ref{sec:hyperparameters}); (iii) inference-time decoding settings (\texttt{greedy}, \texttt{temperature=0}, \texttt{topK=1}), evaluation tool-chain (VLMEvalKit with default prompts and \texttt{gpt-4o-1120} as judge), and hardware (H20-96G) are given in §~\ref{sec:eval_details} and §~\ref{sec:efficiency}; (iv) SDIE and PACS algorithms are formalized in §~\ref{sec:sdie}--\ref{sec:pacs} with full Pareto search protocol in §~\ref{app:pareto_config}. Together these suffice to replicate PACE on the same Qwen3-1.7B backbone.
    

\item {\bf Open access to data and code}
    \item[] Question: Does the paper provide open access to the data and code, with sufficient instructions to faithfully reproduce the main experimental results, as described in supplemental material?
    \item[] Answer: \answerNo{} 
    \item[] Justification: The PACE codebase is currently not publicly available due to company restrictions.

\item {\bf Experimental setting/details}
    \item[] Question: Does the paper specify all the training and test details (e.g., data splits, hyperparameters, how they were chosen, type of optimizer) necessary to understand the results?
    \item[] Answer: \answerYes{}
    \item[] Justification: All training and evaluation details are specified. §~\ref{sec:impl_details} gives the implementation summary (Qwen3-1.7B backbone, $L_{\text{buf}}{=}4$, M-RoPE base frequencies, 5$\times$8 H20-96G GPUs, bfloat16). 
    Full hyperparameters across the four training stages, including resolution schedule, optimizer (AdamW with $\beta_1{=}0.9$, $\beta_2{=}0.999$), weight decay 0.01, cosine LR schedule, peak LR, warmup, training steps, max sequence length, and global batch size, are listed in Table~\ref{tab:hyperparam} (§~\ref{sec:hyperparameters}). Architecture configuration is in Table~\ref{tab:arch_details} (§~\ref{sec:arch_details}). Evaluation protocol (VLMEvalKit defaults, gpt-4o-1120 judge, greedy decoding, MME sum of perception+reasoning, HallusionBench averaged aAcc/qAcc/fAcc) is documented in §~\ref{sec:eval_details} and the table captions of §~\ref{sec:experiments}.

\item {\bf Experiment statistical significance}
    \item[] Question: Does the paper report error bars suitably and correctly defined or other appropriate information about the statistical significance of the experiments?
    \item[] Answer: \answerYes{} 
    \item[] Justification: We run PACS with 3 independently sampled calibration sets (different seeds, 1{,}000 images each) and report mean $\pm$ std on a representative subset of 3 benchmarks covering distinct capability domains (MMB, AI2D and HallBench) across all three retention levels.
    Full results are in Appendix~\ref{app:significance}. Observed std values are consistently smaller than PACE-vs-baseline gaps, confirming statistical significance. 

\item {\bf Experiments compute resources}
    \item[] Question: For each experiment, does the paper provide sufficient information on the computer resources (type of compute workers, memory, time of execution) needed to reproduce the experiments?
    \item[] Answer: \answerYes{} 
    \item[] Justification: Compute resources are fully disclosed. §~\ref{sec:impl_details} and §~\ref{sec:hyperparameters} state that all training stages use 5$\times$8 H20-96G GPUs with bfloat16 precision; Table~\ref{tab:hyperparam} gives per-stage training steps (PT: 250K, MT: 46K, SFT: 19K, RL: 2K) and global batch sizes (2560 / 1920 / 640 / 128 respectively). Inference efficiency experiments (§3.5) are run on a single H20-96G GPU at batch size 1 and 1024$^2$ input resolution, with per-stage latency, KV-cache, FLOPs, and GPU memory reported in Table~\ref{tab:efficiency}. PACS's Pareto-front traversal takes $\sim$10 GPU-minutes on the same hardware (§~\ref{sec:pareto_appendix}). Preliminary and ablation experiments not reported in the main tables used the same cluster but significantly less compute; we acknowledge this in §~\ref{sec:hyperparameters}.
    
\item {\bf Code of ethics}
    \item[] Question: Does the research conducted in the paper conform, in every respect, with the NeurIPS Code of Ethics \url{https://neurips.cc/public/EthicsGuidelines}?
    \item[] Answer: \answerYes{} 
    \item[] Justification: We have reviewed the NeurIPS Code of Ethics and confirm that the research conforms in every respect. The work does not involve human subjects, crowd-sourcing, or personally identifiable information. All training data are from publicly released academic datasets (listed in Tables~\ref{tab:datasets_pt} -- \ref{tab:datasets_rl}) used under their original licenses; we do not scrape new data. Model evaluation uses publicly available benchmarks. Anonymity is preserved in the current submission.

\item {\bf Broader impacts}
    \item[] Question: Does the paper discuss both potential positive societal impacts and negative societal impacts of the work performed?
    \item[] Answer: \answerYes{} 
    \item[] Justification: Positive impacts: PACE reduces FLOPs by $\sim$2$\times$ and KV-cache by $\sim$3.5$\times$ on native MLLMs (§~\ref{sec:efficiency}) without sacrificing---and occasionally improving---benchmark accuracy, making multimodal inference more accessible on resource-constrained devices and reducing the energy cost of large-model serving. It also accelerates the adoption of the native/ViT-free MLLM paradigm, which removes the brittle modality-boundary assumption of modular MLLMs. Potential negative impacts: as a general efficiency technique, PACE inherits the societal risks of any performant MLLM, including potential misuse for automated disinformation, non-consensual surveillance, or biased decision-making. PACE itself neither amplifies nor mitigates these risks beyond the underlying backbone. We do not introduce a new generative model with direct deployment applications beyond the already-public Qwen3-1.7B-based backbone.
    
\item {\bf Safeguards}
    \item[] Question: Does the paper describe safeguards that have been put in place for responsible release of data or models that have a high risk for misuse (e.g., pre-trained language models, image generators, or scraped datasets)?
    \item[] Answer: \answerNA{} 
    \item[] Justification: PACE is a visual-token compression framework applied to an existing public backbone. It does not introduce new generative capabilities, does not scrape or release new datasets, and does not increase the risk profile of the underlying model beyond what is already released in the public domain. Consequently, no additional safeguards beyond those already provided by the underlying backbone are required. Upon code release, we will follow standard licensing and usage-guideline practices aligned with Qwen3's license.

\item {\bf Licenses for existing assets}
    \item[] Question: Are the creators or original owners of assets (e.g., code, data, models), used in the paper, properly credited and are the license and terms of use explicitly mentioned and properly respected?
    \item[] Answer: \answerYes{} 
    \item[] Justification: All existing assets are properly credited. The Qwen3 backbone~\cite{bai2025qwen3, yang2025qwen3} is used under its original Apache-2.0 license. Training data sources, including all datasets in Table~\ref{tab:datasets_pt}--~\ref{tab:datasets_rl}, are each cited to their original publications with their respective licenses respected. Evaluation benchmarks (MMB, AI2D, MME, HalBench, SEED, RWQA, MMStar, MathVista) are cited in Table~\ref{tab:main_results}--~\ref{tab:efficiency}. The VLMEvalKit~\cite{duan2024vlmevalkit} tool-chain is used under its public license.

\item {\bf New assets}
    \item[] Question: Are new assets introduced in the paper well documented and is the documentation provided alongside the assets?
    \item[] Answer: \answerNo{} 
    \item[] Justification: The PACE codebase, model checkpoint, and accompanying documentation are currently not publicly available due to company restrictions.

\item {\bf Crowdsourcing and research with human subjects}
    \item[] Question: For crowdsourcing experiments and research with human subjects, does the paper include the full text of instructions given to participants and screenshots, if applicable, as well as details about compensation (if any)? 
    \item[] Answer: \answerNA{} 
    \item[] Justification: The paper does not involve crowd-sourcing or research with human subjects. All datasets used are publicly available academic benchmarks. 
    No new human annotations are collected in this work.

\item {\bf Institutional review board (IRB) approvals or equivalent for research with human subjects}
    \item[] Question: Does the paper describe potential risks incurred by study participants, whether such risks were disclosed to the subjects, and whether Institutional Review Board (IRB) approvals (or an equivalent approval/review based on the requirements of your country or institution) were obtained?
    \item[] Answer: \answerNA{} 
    \item[] Justification: The paper does not involve human-subjects research, so IRB approval is not applicable.

\item {\bf Declaration of LLM usage}
    \item[] Question: Does the paper describe the usage of LLMs if it is an important, original, or non-standard component of the core methods in this research? Note that if the LLM is used only for writing, editing, or formatting purposes and does \emph{not} impact the core methodology, scientific rigor, or originality of the research, declaration is not required.
    \item[] Answer: \answerNo{} 
    \item[] Justification: LLMs are not used as an original, non-standard component of the core methodology.

\end{enumerate}

\end{document}